\documentclass[11pt,dvipsnames]{article}
\pdfoutput=1

\usepackage{tikz}
\newbox\tempbox
\newif\ifdebug
\debugtrue

\makeatletter
\def\pgfmathsetlength #1#2{\expandafter \pgfmath@onquick #2\pgfmath@ 
  {\ifdebug\typeout{HERE\detokenize{#1}\detokenize{#2}}\fi
   \setbox\tempbox\hbox{\pgfmath@selectfont#1#2\relax\expandafter}\expandafter#1\the #1\relax 
   \ifdebug\typeout{#1=\the#1}\fi}%
  {\pgfmathparse {#2}\ifpgfmathmathunitsdeclared #1\pgfmathresult mu\relax 
                     \else #1\pgfmathresult pt\relax \fi }\ignorespaces }
\makeatother

\usetikzlibrary{positioning}
\usetikzlibrary{arrows.meta}
\usetikzlibrary{bending}
\usepackage{booktabs}
\usepackage{pgfplots}
\usepackage{graphicx}
\usepackage{tikz}
\usepackage{amsmath, amsthm}
\usepackage{amssymb}
\usepackage{graphicx}
\usepackage{setspace}
\usepackage{xfrac}
\usepackage{hyperref}
\usepackage{xstring}
\usepackage{xspace}
\usepackage{bm}
\usepackage{microtype}
\usepackage{graphicx}

\usepackage{comment}
\pgfplotsset{compat=1.14}

\usepackage{gensymb}
\usepackage{hyperref}       
\usepackage{amsfonts}       
\usepackage{nicefrac}       
\usepackage{microtype}      
\usepackage{amsmath}
\usepackage{amsfonts}
\usepackage{amssymb}
\usepackage{comment}
\usepackage[ruled,vlined]{algorithm2e}
\usepackage{hyperref}
\usepackage{amsthm}
\usepackage{amscd}
\usepackage{t1enc}
\usepackage{enumerate}
\usepackage{enumitem}
\usepackage[mathscr]{eucal}
\usepackage{indentfirst}
\usepackage{listings}
\usepackage{epic}
\usepackage{epstopdf} 
\usepackage{wrapfig}

\usepackage{float}
\allowdisplaybreaks

\usepackage{microtype}
\usepackage{graphicx}
\usepackage{subfigure}
\usepackage{booktabs} 

\usepackage{hyperref}

\usepackage{booktabs}
\usepackage{pgfplots}
\usepackage{amsmath, amsthm}
\usepackage{amssymb}
\usepackage{graphicx}
\usepackage{setspace}
\usepackage{xfrac}
\usepackage{hyperref}
\usepackage{xstring}
\usepackage{xspace}
\usepackage{bm}
\usepackage{tikz}

\usetikzlibrary{arrows.meta}
\usepackage{microtype}
\usepackage{graphicx}

\usepackage{booktabs} 

\usepackage{comment}
\pgfplotsset{compat=1.14}

\usepackage{gensymb}
\usepackage{hyperref}       
\usepackage{amsfonts}       
\usepackage{nicefrac}       
\usepackage{microtype}      
\usepackage{amsmath}
\usepackage{amsfonts}
\usepackage{amssymb}
\usepackage[ruled,vlined]{algorithm2e}

\usepackage{hyperref}

\usepackage{amsthm}
\usepackage{amscd}
\usepackage{t1enc}
\usepackage{enumerate}
\usepackage{enumitem}
\usepackage[mathscr]{eucal}
\usepackage{indentfirst}
\usepackage{listings}
\usepackage{graphicx}
\usepackage{graphics}
\usepackage{pict2e}
\usepackage{epic}
\usepackage{epstopdf} 
\usepackage{wrapfig}
\usepackage{comment}
\usepackage{fullpage}

\newcommand{\bE}{\mathbb{E}}

\newcommand{\bI}{\mathbb{I}}

\renewcommand{\tilde}{\widetilde}
\renewcommand{\nu}{\vartheta}

\newcommand{\defl}{L_{\mathrm{def}}^{0-1}}
\newcommand{\arginf}{\mathrm{arginf}}
\newtheorem{theorem}{Theorem}
\newtheorem*{theorem*}{Theorem}
\newtheorem{lemma}{Lemma}
\newtheorem{proposition}{Proposition}

\newtheorem*{examples}{Examples}

\newtheorem{Assum}{Assumption}

\usepackage{hyperref}
\usepackage{environ}
\usepackage{tikz}
\usepackage{float}

\usepackage{math_style_amin}

\date{}
\title{\textbf{Sample Efficient Learning of Predictors that Complement Humans}}
\author{Mohammad-Amin Charusaie  \thanks{Equal Contribution.}  \thanks{Max Planck Institute for Intelligent Systems. Email \texttt{mcharusaie@tuebingen.mpg.de}}   \and Hussein Mozannar $^ *$\thanks{Massachusetts Institute of Technology. Email: \texttt{mozannar@mit.edu}} \and David Sontag \thanks{Massachusetts Institute of Technology. Email: \texttt{dsontag@csail.mit.edu}} \and Samira Samadi \thanks{Max Planck Institute for Intelligent Systems. Email \texttt{ssamadi@tuebingen.mpg.de}}}

\begin{document}
\maketitle
\begin{abstract}

One of the goals of learning algorithms is to complement and reduce the burden on human decision makers. The expert deferral setting wherein an algorithm can either predict on its own or defer the decision to a downstream expert helps accomplish this goal.  A fundamental aspect of this setting is the need to learn complementary predictors that improve on the human's weaknesses rather than learning predictors optimized for average error. In this work, we provide the first theoretical analysis of the benefit of learning complementary predictors in expert deferral. To enable efficiently learning such predictors, we consider a family of consistent surrogate loss functions for expert deferral and analyze their theoretical properties. Finally, we design active learning schemes that require  minimal amount of data of human expert predictions in order to learn accurate deferral systems. 

\end{abstract}
\section{Introduction}
How do we combine AI systems and human decision makers to both reduce error and alleviate the burden on the human?
AI systems are starting to be frequently used in combination with human decision makers, including in high-stakes settings like healthcare \cite{beede2020human} and content moderation \cite{gillespie2020content}. A possible way to combine the human and the AI is to learn a 'rejector' that queries either the human or the AI to predict on each input. This allows us to route examples to the AI model, where it outperforms the human, so as to simultaneously reduce error and human effort. Moreover, this formulation allows us to jointly optimize the AI  so as to complement the human's weaknesses, and to optimize the rejector to allow the AI to defer when it is unable to predict well. 
This type of interaction is typically referred to as \emph{expert deferral} and the learning problem is that of jointly learning the AI classifier and the rejector. Empirically this approach has been shown to outperform either the human or the AI when predicting by their own \cite{kamar2012combining,tan2018investigating}. One hypothesis is that humans and machines make different kinds of errors. For example humans may have bias on certain features \cite{kleinberg2018human} while AI systems may have bounded expressive power or limited training data. On the other hand, humans may outperform AI systems as they may have side information that is not available to the AI, for example due to  privacy constraints.

Existing deployments tend to ignore that the system has two components: the AI classifier (henceforth, the classifier)  and the human.
Typically the AI is trained without taking int account the human---and deferral is done using post-hoc approaches like model confidence \cite{raghu2019algorithmic}.
The main problem of this approach, that we refer to as \emph{staged learning}, is that it ignores the possibility of learning a better combined system by accounting for the human (and its mistakes) during training. More recent work has developed joint training strategies for the AI and the rejector based on surrogate losses and alternating minimization \cite{mozannar2020consistent, okati2021differentiable}. However, we lack  a theoretical understanding of the fundamental merits of joint learning compared to the staged approach. In this work, we study three main challenges in expert deferral from a theoretical viewpoint: 1) \emph{model capacity} constraints, 2) lack of \emph{data of human expert's prediction} and 3)  optimization using \emph{surrogate losses}. 

When learning a predictor and rejector in a limited hypothesis class, it becomes more valuable to allocate model capacity to complement the human. We prove a bound on the gap between the approach that learns a predictor that complements human and the approach that learns the predictor ignoring the presence of the human in Section \ref{sec: staged_learning}. To practically learn to complement the human, the literature has shown that surrogate loss functions are successful \cite{madras2018predict,mozannar2020consistent}. We propose a family of surrogate loss functions that generalizes existing surrogates such as the surrogate in \cite{mozannar2020consistent}, and we further prove  surrogate excess risk bounds and generalization properties of these surrogates in Section \ref{sec: surrogate}. Finally, a main limitation of being able to complement the human is the availability of samples of human predictions. For example, suppose we wish to deploy a system for diagnosing pneumonia from chest X-rays in a new hospital. To be able to know when to defer to the new radiologists, we need to understand their specific strengths and weaknesses. We design a provable active learning scheme that is able to first understand the human expert error boundary and learn a classifier-rejector pair that adapts to it in Section \ref{sec: active}. To summarize, the contributions of this paper are the following:
\begin{itemize}
    

    \item \textbf{Understanding the gap between joint and staged learning}: we prove bounds on the gap when learning in bounded capacity hypothesis classes and with missing human data.
    
    \item \textbf{Theoretical analysis of Surrogate losses:} we propose a novel family of consistent surrogates that generalizes prior work and analyze asymptotic and sample properties.
    
        \item \textbf{Actively learning to defer:} we provide an algorithm that is able to learn a classifier-rejector pair by minimally querying the human on selected points.
\end{itemize}

\section{Related Work}
A growing literature has focused on building models that can effectively defer predictions to human experts. Initial work posed the problem as that of a mixture of experts \cite{madras2018predict}, however, their approach  does not allow the model to adapt to the expert. A different natural baseline that is proposed in \cite{raghu2019algorithmic} learns a predictor that best classifies the target and then the compare its confidence to that of the expert. This is what we refer to as \emph{staged learning} and in our work we provide the first theoretical results on the limitations of this approach. \cite{wilder2020learning} and \cite{pradier2021preferential} jointly learn a classifier and rejector based on the mixture of experts loss, but the method lacks a theoretical understanding and requires heuristic adjustments. \cite{mozannar2020consistent} proposes the first consistent surrogate loss function for the expert deferral setting which leads to an effective joint learning approach with subsequent work building on their approach \cite{raman2021improving,liu2021incorporating}. In this paper, we generalize the surrogate presented in \cite{mozannar2020consistent}  and present generalization guarantees that enable us to effectively bound performance when learning with this surrogate. \cite{keswani2021towards} proposes a surrogate loss  which is the sum of the loss of learning the classifier and rejector separately but which is not a consistent surrogate. \cite{okati2021differentiable} proposes an iterative method that alternates between optimizing the predictor and the rejector and show that it converges to a local minimum and empirically  matches the performance of the surrogate in \cite{mozannar2020consistent}. Multiple works have used the learning-to-defer paradigm in other settings \cite{joshi2021pre,gao2021human,zhao2021directing,straitouri2021reinforcement}.

In our work, we derive an active learning scheme that enables us to understand the human expert error boundary with the least number of examples. This bears similarity to work on onboarding humans on AI models where the objective is reversed: teaching the human about the AI models error boundary \cite{ribeiro2016should,lai2020chicago, mozannar2021teaching} and work on machine teaching \cite{su2017interpretable,zhu2018overview}. However, our setting requires distinct methodology as we have no restrictions on the parameterization of our rejector  which the previous line of work assumes. Works on Human-AI interaction usually keep the AI model fixed and optimize for other aspects of the interaction, while in our work we optimize the AI to complement the human \cite{kerrigan2021combining,bansal2019updates}.

The setting when the cost of deferral is constant  has a long history in machine learning and goes by the name of rejection learning  \cite{cortes2016learning, chow1970optimum,bartlett2008classification,charoenphakdee2021classification} or selective classification (only predict on x\% of data) \cite{el2010foundations,geifman2017selective,gangrade2021selective,acar2020budget}.  \cite{shah2020online} explored an online active learning scheme for rejection learning, however, their scheme was tailored to a surrogate for rejection learning that is not easily extendable to expert deferral.  Our work also bears resemblance to active learning with weak (the expert) and strong labelers (the ground truth) \cite{zhang2015active} 


\section{Problem Setting}

We study classification problems where the goal is to  predict a target $Y \in \{1,\cdots,K\}$ based on a set of features $X \in \mathcal{X}$, or via querying a human expert opinion $M\sim \mu_{M|XY}$ that has access to a domain $\mathcal{Z}$. Upon viewing the input $X$, we decide first via a rejector function  $r: \mathcal{X} \to \{0, 1 \}$ whether to defer to the expert, where $r(\xv)=1$ means deferral and $r(\xv)=0$ means predicting using a classifier $h: \mathcal{X} \to [K]$.
The expert domain may contain side information beyond ${X}$ to classify instances. For example, when diagnosing diseases from chest X-rays the human may have access to the patient's medical records while the AI only has access to the X-ray. We assume that $X, Y, M$ have a joint probability measure $\mu_{XYM}$.

We let deferring the decision to the expert incur a cost equal to the expert's error and an additional penalty term: $\ell_{\exp}(\xv,y,m) = \bI_{m \neq y} + c_{\exp}(\xv,y,m)$ that depends on the features $\xv$,  the value of target $Y=y$, and the expert's prediction $M=m$. Moreover, we assume that predicting without querying the expert incurs a different cost equal to the classifier error and an additional penalty: $\ell_{\textrm{AI}}(\xv,y,m) = \bI_{h(\xv) \neq y} + c_{\textrm{AI}}(\xv,y,m)$ where $h(\xv)$ is the prediction of the classifier. With the above in hand, we write the true risk as
\begin{align}
    L_{\mathrm{def}}(h,r)=  \bE_{X, Y, M} \ [ \ \ell_{\textrm{AI}}\big(X,Y,h(X)\big) \cdot \bI_{r(X) = 0}   +  \ell_{\exp}(X,Y,M) \cdot \bI_{r(X)=1} \  ]  \label{eq:original_reject_loss} 
\end{align}
In the setting when we only care about misclassification costs with no additional penalties, the deferral loss becomes a $0-1$ loss as follows:
\begin{align}
{L_{\mathrm{def}}^{0{-}1}(h,r)= \label{eq:01_reject_loss}  \bE \ [ \ \bI_{h(X) \neq Y} \bI_{r(X) = 0} +  \bI_{M \neq Y}\bI_{r(X)=1} \  ]}
\end{align}
We focus primarily on the $0-1$ loss for our analysis; it is also possible to extend parts of the analysis to handle additional cost penalties. 
We restrict our search to  classifiers within a hypothesis class $\mathcal{H}$ and a rejector function within a hypothesis class $\mathcal{R}$. The optimal joint classifier and rejector pair is the one that minimizes \eqref{eq:01_reject_loss}:
\begin{equation}
h^*,r^* = \argmin_{h \in \mathcal{H}, r \in \mathcal{R}} L_{\mathrm{def}}^{0-1}(h,r) \label{eq:joint_opt_problem}
\end{equation}
To approximate the optimal classifier-rejector pair, we have to handle two main obstacles: (i)  \emph{optimization} of the non-convex and discontinuous loss function and (ii) availability of the \emph{data} on human's predictions and the true label. 

In the following section \ref{sec: staged_learning} and in section \ref{sec: active}, we restrict the analysis to binary labels $Y = \{0,1\}$ for a clearer exposition. The theoretical results in the following section are shown to apply further for the multiclass setting in a set of experimental results. However, in section \ref{sec: surrogate}, where we discuss practical algorithms, we switch back to the mutliclass setting for full generality.
In the following section, we compare two   strategies for expert deferral across these two dimensions.



\section{Staged Learning of Classifier and Rejector}\label{sec: staged_learning}
\subsection{Model Complexity Gap}
\paragraph{Staged learning.} The optimization problem framed in  \eqref{eq:joint_opt_problem} requires joint learning of the classifier and rejector. Alternatively, a popular approach comprises of first learning a classifier that minimizes average misclassification error on the distribution, and then, learning a rejector that defers each point to either classifier or the expert, depending on who has a lower estimated error \cite{raghu2019algorithmic,wilder2020learning}.

Formally, we first learn $h$ to minimize the average misclassification error:
\begin{equation}
\hat{h} = \argmin_{h \in \mathcal{H}} \mathbb{E}_{X, Y}[ {\bI}_{h(X) \neq Y} ] \label{eqn: seq_1}
\end{equation}
and in the second step we learn the rejector $r$ to minimize the joint loss \eqref{eq:01_reject_loss} with the now fixed classifier $\hat{h}$:
\begin{equation}
\hat{r} = \argmin_{r \in \mathcal{R}} L_{\mathrm{def}}^{0 - 1}(\hat{h}, r) \label{eqn: seq_2}
\end{equation}
This procedure is particularly attractive as the two steps \eqref{eqn: seq_1} and \eqref{eqn: seq_2} could be cast as  classification problems, and approached by powerful known tools that are devised for such problems. Despite its convenience, this method is not guaranteed to achieve the optimal loss (as in \eqref{eq:joint_opt_problem}), since it decouples the joint learning problem. Assuming that we are able to optimally solve both problems on the true distribution, let $(h^*,r^*)$ denote the solution of joint learning and $(\hat{h},\hat{r})$ the solution of staged learning. To estimate the extent to which  staged learning is sub-optimal,  we define the following minimax measure $\De(d_1, d_2)$ for the binary label setting: 
\begin{equation*}
\De(d_1, d_2)=\inf_{\Hcal, \Rcal\in \mathfrak{H}_{d_1, d_2}} \sup_{\mu_{XYM}} L_{\mathrm{def}}^{0 - 1}(\hat{h}, \hat{r})-L_{\mathrm{def}}^{0 - 1}(h^*, r^*)
\end{equation*}
To disentangle the above measure, the supremeum $\sup_{\mu_{XYM}}$ is a worst-case over the data distribution and expert pair, while the infimum $\inf_{\Hcal, \Rcal\in \mathfrak{H}_{d_1, d_2}}$ is the best-case classifier-rejector model classes with specified complexity $d_1$ and $d_2$
where $\mathfrak{H}_{d_1, d_2} = \{(\Hcal, \Rcal)\,:\, d(\Hcal)=d_1, d(\Rcal)=d_2\}$ and $d(\cdot)$ denotes the VC dimension of a hypothesis class.
As a result, this measure expresses the  worst-case gap between joint and staged learning when learning from the optimal model class given complexity of the predictor and rejector model classes.
The following theorem provides a lower- and upper-bound on $\De(d_1, d_2)$.

\begin{theorem} \label{thm: vc}
    For every set of hypothesis classes $\Hcal, \Rcal$  where $d(\cdot)$ denotes the VC-dimension of a hypothesis class, the minimax difference measure between joint and staged learning is bounded between:
\begin{equation}
    \frac{1}{d(\Hcal)+1}\leq\De(d(\Hcal), d(\Rcal))\leq \frac{d(\Rcal)}{d(\Hcal)}
\end{equation}
\end{theorem}
Proof of the theorem can be found in Appendix \ref{app: vc}.
The theorem implies that for any classifier and rejector hypothesis classes, we can find a distribution and an expert such that the gap between staged learning and joint learning is at least $1$ over the VC dimension of the classifier hypothesis class. Meaning the more complex our classifier hypothesis class is, the smaller the gap between joint and staged learning is. On the other hand, the gap is no larger than the ratio between the rejector complexity over the the classifier complexity. Which again implies if our hypothesis class is comparatively much  richer than the rejector class, the gap between the joint and staged learning reduces.  What this does not mean is that deferring to the human is not required for optimal error when the classifier model class is very large, but that training the classifier may not require knowledge of the human performance.


\subsection{Data Trade-offs}\label{subsec:data_tradeofs}

Current datasets in machine learning are growing in size and are usually of the form of feature $X$ and target $Y$ pairs. It is unrealistic to assume that the human expert is able to individually provide their predictions for all of the data. In fact, the collection of datasets in machine learning often relies on crowd-sourcing where the label can either be a majority vote of multiple human experts, e.g. in hate-speech moderation \cite{davidson2017automated}, or due to an objective measurement,  e.g. a lab test result for a patient medical data. In the expert deferral problem, we are interested in the predictions of a particular human expert and thus it is infeasible for that human to label all the data and perhaps unnecessary. 

In the following analysis, we assume access to fully labeled data $S_l = \{(\xv_i,y_i,m_i)\}_{i=1}^{n_l}$ and data without expert labels $S_u=\{ (\xv_i,y_i)\}_{i=n_l+1}^{n_l+n_u}$. This is a realistic form of the data we have available in practice. We now try to understand how we can learn a  classifier and rejector from these two datasets. This is where we expect the staged learning procedure can become attractive as it can naturally exploit the two distinct datasets to learn.

\paragraph{Joint Learning.} Learning jointly requires access to the dataset with the entirety of expert labels, thus we can only use $S_l$ to learn
\begin{equation*}
\tilde{h}, \tilde{r} = \argmin_{h,r} \sum_{i \in S_l} \bI_{h(\xv_i) \neq y_i} \bI_{r(\xv_i) = 0}  + \bI_{y_i \neq m_i} \bI_{r(\xv_i)=1}
\end{equation*}
\paragraph{Staged learning.} On the other hand, for staged learning we can exploit our expert unlabeled data to first learn $h$:
\begin{equation*}
\hat{h} = \argmin_{h} \sum_{i \in S_u} {\bI}_{h(\xv_i) \neq y_i} 
\end{equation*}
and in the second step we learn $\hat{r}$ to minimize the joint loss with the fixed $\hat{h}$ but only on $S_l$.

\paragraph{Generalization.} Given that staged learning exploits both datasets, we expect that if we have much more expert unlabeled data than labeled data, i.e. $n_u \gg n_l$, then it maybe possible to  obtain better generalization guarantees from staged learning. The following proposition shows that when the Bayes optimal classifier is in the hypothesis class, then staged learning can possibly improve sample complexity over joint learning.

\begin{proposition} \label{prop: gen_labeled_unlabeled}
Let $\Scal_{l}=\{(\xv_i, y_i, m_i)\}_{i=1}^{n_l}$ and $\Scal_{u}=\{(\xv_{i+n_l}, y_{i+n_l})\}_{i=1}^{n_u}$ be two iid sample sets that are drawn from the distribution $\mu_{XYM}$ and are labeled and not labeled by the human, respectively. 
Assume that the optimal classifier $\bar{h}=\argmin_{h}\Ebb_{X, Y\sim \mu_{XY}}[\bI_{h(X)\neq Y}]$ is a member of $\Hcal$ (i.e., realizability).

Let $(\hat{h}, \hat{r})$ be the staged solution and let   $(\tilde{h}, \tilde{r})$ be the joint solution obtained by learning only on $S_L$. 
Then, with probability at least $1-\delta$ we have for staged learning
\begin{align}
L_{\mathrm{def}}^{0 - 1}(\hat{r}, \hat{h})\leq
	&L_{\mathrm{def}}^{0 - 1}(h^*,r^*)+ \bm{\Rad_{n_u}(\Hcal)}+2\Rad_{n_l}(\Rcal) \label{eqn: gen_labeled_unlabeled}	+2\min\big\{P(M\neq Y), \Rad_{n_l\Pr( M\neq Y)/2}(\Rcal)\big\}\nonumber\\&+C\sqrt{\frac{\log 1/\delta}{\min(n_l,n_u)}}+P(M\neq Y) e^{-n_{l}\frac{\Pr(M\neq Y)^2}{2}} 
\end{align}
while for joint learning we have:
\begin{align}
L_{\mathrm{def}}^{0 - 1}(\tilde{r}, \tilde{h})\leq&
	L_{\mathrm{def}}^{0 - 1}(h^*,r^*)+\bm{\Rad_{n_l}(\Hcal)}+2\Rad_{n_l}(\Rcal)+  2\Rad_{{n_l\small{\Pr(M\neq Y)/2}}}(\Rcal)+C\sqrt{\frac{\log 1/\delta}{n_l}}\nonumber\\&+ P(M\neq Y) e^{-n_l\tfrac{\Pr(M\neq Y)}{2}} 
\end{align}
\end{proposition}
Proof of the proposition can be found in Appendix \ref{app: gen_labeled_unlabeled}. From the above proposition, when the Bayes classifier is in the hypothesis class, the upper bound for the sample complexity required to learn the classifier and rejector is reduced by only paying the Rademacher complexity of the hypothesis class on the unlabeled data compared to on the potentially smaller labeled dataset. The Rademacher complexity is a measure of model class complexity on the data and can be related to the VC dimension.

While in this case study staged learning may improve the generalization error bound comparing to that of joint learning, the number of labeled samples for both to achieve $\ep$-upper-bound on  the true risk is of order $O(\tfrac{\log 1/\ep}{\ep^{2}})$. 
As we can see, there exist computational and statistical trade-offs between joint and staged learning. While joint learning leads to more accurate systems, it is computationally harder to optimize than staged learning. In the next section, we investigate whether it is possible to more efficiently solve the joint learning problem while still retaining its favorable guarantees in the multiclass setting.

\section{Surrogate Losses For Joint Learning}\label{sec: surrogate}
\subsection{Family of Surrogates}
A common practice in machine learning is to propose surrogate loss functions, which often are continuous and convex, that approximate the original loss function we care about \cite{bartlett2006convexity}. The hope is that these surrogates are more readily optimized and minimizing them leads to predictors that also minimize the original loss. In their work on expert deferral,
 \cite{mozannar2020consistent} reduces the learning to defer problem to cost-sensitive learning which enables them to use surrogates for cost-sensitive learning in the expert deferral setting. We follow the same route in deriving our novel family of surrogate losses. 
 We now recall the reduction in \cite{mozannar2020consistent}:
define the random costs $\mathbf{c} \in \mathbb{R}_+^{K+1}$ where  $c(i)$ is the $i'$th component of $\mathbf{c}$ and represents the cost of predicting  label $i \in [K+1]$. The goal of cost sensitive learning is to build a predictor $\tilde{h}: \mathcal{X} \to [K+1]$ that minimizes the cost-sensitive loss $\Ebb[c(\tilde{h}(X))]$. 
The reduction is accomplished by setting 
 $c(i)=\ell_{\textrm{AI}}\big(X, Y, i\big)$ for $i \in [K]$ while $c(K+1)$ represents the cost of deferring to the expert with $c(K+1)=\ell_{\mathrm{exp}}(X, Y, M)$. Thus, the predictor $\tilde{h}$ learned in cost-sensitive learning  implicitly defines a classifier-rejector pair $(h,r)$  with the following encoding:  
 \begin{equation}
h(\xv),r(\xv) =
 \begin{cases}
   h(\xv)=i, r(\xv)=0,   & \text{if} \ \tilde{h}(\xv) = i \in [K]  \\
    h(\xv) = 1, r(\xv)=1  & \text{if} \ \tilde{h}(\xv) = K+1 
 \end{cases} \label{eq: decoding_hr}
 \end{equation}
Note that when $\tilde{h}(\xv) = K+1$ the classifier $h$ is left unspecified and thus we assign it a dummy value of $1$. 
Cost-sensitive learning is a non-continuous and non-convex optimization problem that makes it computationally hard to solve in practice. In order to approximate it, we propose a novel family of cost-sensitive learning loss functions that extend any multi-class loss function to the cost-sensitive setting. 

First we parameterize our predictor $\tilde{h}$ with $K+1$ functions $f_i: \mathcal{X} \to \mathbb{R}$ and define the predictor to be the max of these $K+1$ functions: $\tilde{h}_{\fv}(\xv) = \arg\max_{i} f_i(\xv)$. Note that $\tilde{h}_{\fv}$ gives rise to the classifier-rejector pair $(h_{\fv},r_{\fv})$ according to the decoding rule \eqref{eq: decoding_hr}.

Formally, let $\ell_{\phi}(y,\mathbf{f}(\xv)):[K+1]\times \mathbb{R}^{K+1} \to \Rbb$ be a surrogate loss function of the zero-one loss for  multi-class classification. We define the extension of this surrogate to the cost-sensitive setting as:
\begin{equation}
\tilde{\ell}_{ \phi}(\mathbf{c}, \mathbf{f}(\xv)) = \sum_{i=1}^{K+1}\big[\max_{j\in{k+1}} c(j)-c(i)\big]\ell_{\phi}( i, \mathbf{f}(\xv)) \label{eq: surrogate_costsensitive}
\end{equation}
Note that if $\ell_{\phi}$ is continuous or convex, then because $\ell_{\cv, \phi}$ is a finite positively-weighted sum of $\ell_{\phi}$'s, then $\ell_{\cv, \phi}$ is also continuous or convex, respectively.
We show in the following proposition, that if $\ell_{\phi}$ is a consistent surrogate for multi-class classification, then $\tilde{\ell}_{\phi}$ is consistent for cost-sensitive learning and by the reduction above is also consistent for learning to defer.

\begin{proposition} \label{prop: consistency}
Suppose $\ell_{\phi}(y,\mathbf{f}(\xv))$ is a consistent surrogate for multi-class classification, meaning if the surrogate is minimized over all functions then it also minimizes the misclassification loss:
 \begin{center}
      let $ \bm{f}^* =\arginf_{\mathbf{f}} \bE\left[ \ell_{\phi}(Y,\mathbf{f}(\xv)) \right]$, then:  $\tilde{\hv}_{f^*} = \arginf_{\hv} \Ebb[\mathbb{I}_{Y \neq {\hv}(X)}]$, where $\tilde{\hv}_{f^*}$ is defined as above.
 \end{center}
Then, our surrogate $\tilde{\ell}_{ \phi}(\mathbf{c}, \mathbf{g}(\xv))$ defined in \eqref{eq: surrogate_costsensitive} is a consistent surrogate for cost-sensitive learning and thus for learning to defer:
 \begin{center}
      let $\bm{\tilde{f}}^*=\arginf_{\mathbf{g}} \bE\left[ \tilde{\ell}_{\phi}(\mathbf{c},\mathbf{f}(\xv)) \right]$, then:  $h_{\fv}^*,r_{\fv}^* = \arginf_{h,r} L_{\mathrm{def}}^{0 - 1}(h,r) $, with   $(h_{\fv}^*,r_{\fv}^*)$ defined  in \eqref{eq: decoding_hr}
 \end{center}
 \label{prop: consistent_costsensitive}
\end{proposition}
Proof of the proposition can be found in Appendix \ref{app: consistency}. 
To illustrate the family of surrogates implied by Proposition \ref{prop: consistent_costsensitive}, we first start by recalling a family of surrogates for multi-class classification. 
 Theorem 4 of \cite{zhang2004statistical} shows that there is a family of  consistent surrogates for $0-1$ loss in multi-class classification parameterized by three functions  $(u, s, t)$ 
 and takes the form $\ell_{\phi}(y, \fv(\xv))=u(f_{y}(\xv))+s\Big(\sum_{j=1}^{K+1}t\big(f_j(x)\big)\Big)$. This family is consistent under certain conditions of the aforementioned functions. 

Now we show with a few of examples that this family encompasses some popular surrogates used in cost sensitive learning:
\begin{examples} \noindent
	(1) \ If we set $u(x)=-2x$, $s(x)= x$, and $t(x)= x^2$, then we can obtain a weighted quadratic loss:
	\begin{equation}
	    \tilde{L}_{2}= \Ebb[\pi\sum_{i=1}^{K+1}|f_i - q(i)|^2],
	\end{equation}
where $q(i)$ is the normalized expected value of $\max_{j\in [K+1]} c(j)-c(i)$ given $X=x$, and $\pi$ represents the normalization term.

	(2) \ If we set $u(x)=-x$, and $s(x)=\log (x)$ and $t(x)=e^x$, then we have $a\psi'(x)+t'(x)=-a+e^x=0$, and as a result $x=\log a$, which is an increasing function of $a$. As a result, the surrogate loss
	\begin{equation} 
\tilde{L}_{CE}(\fv) = \Ebb[-\sum_{i=1}^{K+1} \big(\max_{j} c(j)- c(i)\big) \log \frac{e^{f_i(X)}}{\sum_{k=1}^{K+1} e^{f_k(X)}} ]\label{eqn:deferral_loss}
 	\end{equation}
	which is the loss defined in \cite{mozannar2020consistent} and used for learning to defer.
\end{examples}

\subsection{Theoretical Properties of Surrogate }

\paragraph{Goodness of a Surrogate.} Given a surrogate, how can we quantify how well it approximates our original loss? One avenue is through the surrogate excess-risk bound as follows.
Let $\tilde{L}$ be a surrogate for the loss function $L$, and let $\tilde{h}^*$ be the minimizer of the surrogate and $h^*$ the minimizer of $L$. 
 We call the \emph{excess surrogate risk} \cite{bartlett2006convexity} the following quantity if we can find a  \emph{calibration function} $\psi$ such that for any $h$ we have:
\begin{equation}
    	\psi(L(h) - L(h^*)) \leq \tilde{L}(h) - \tilde{L}(\tilde{h}^*)
\end{equation}
The excess surrogate risk bound tells us if we are $\ep$-close to the minimizer of the surrogate, then we are $\psi^{-1}(\ep)$-close to the minimizer of the original loss.

We now show that the family of surrogates defined in \eqref{eq: surrogate_costsensitive} has a polynomial excess-risk bound and furthermore prove an excess-risk bound for the surrogate loss function $\tilde{L}_{CE}$ defined in \cite{mozannar2020consistent}.

\begin{theorem} \label{thm: calib_CE}
 Suppose that $\psi(x)=C x^{\ep}$, for $\ep\in [1, \infty)$ is a calibration function for the multiclass surrogate $\ell_{\phi}$ and if $|c(i)|\leq M$ for all $i$, then $\psi'(x)=\frac{C}{M^{\ep-1}}x^{\ep}$ is a calibration function of $\tilde{\ell}_{\phi}(\cv, \cdot)$.
 
 As an example, for the surrogate $\tilde{L}_{CE}$ \eqref{eqn:deferral_loss} the calibration function is $\psi(x)=\frac{1}{16MK}x^2$.
\end{theorem}
Proof of the theorem can be found in Appendix \ref{app: proof_calib_CE}. Note, that \cite{nowak2019general} proved that the for the cross-entropy loss the calibration function $\phi$ is of order $\Theta(\ep^2)$ which is in accordance with our results.

\paragraph{Generalization Error.} Equipped with the excess surrogate risk bound, we can now study the sample complexity properties of minimizing the surrogates proposed. For concreteness, we focus on the surrogate $\tilde{L}_{CE}$ of \cite{mozannar2020consistent} when reduced to the learning to defer setting. The following theorem proves a generalization error bound when minimizing the surrogate $\tilde{L}_{CE}$ for learning to defer.

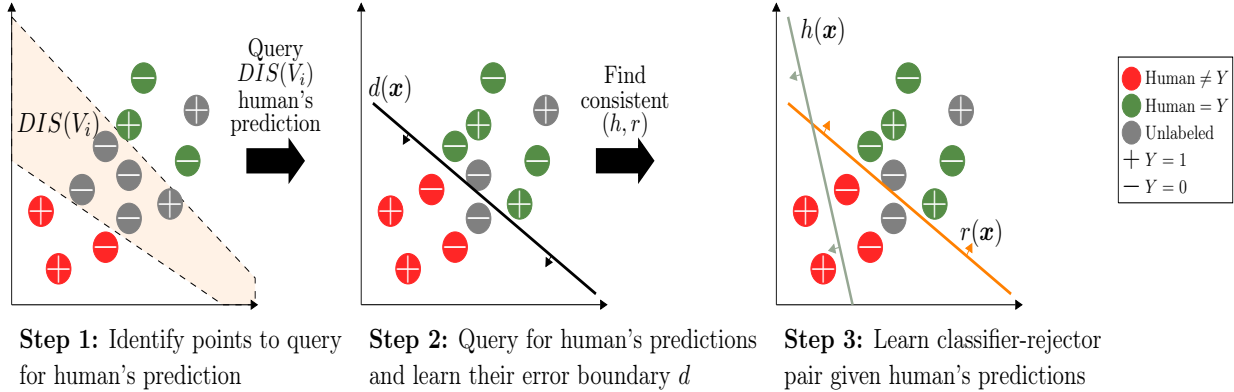
\begin{figure*}
\centering
\resizebox{1.0\textwidth}{150pt}{
\begin{tikzpicture}
\draw[fill=orange!10, dashed] (0, 4) -- (4.15, -16/4.6+4-0.15)  --(4.15, 0)-- (3.6, 0) -- (0, 2) -- (0, 4);
\draw [line width = 0.4 pt, -{Triangle[width=4pt]}] (0, 0) --(4.2, 0);
\draw[ -{Triangle[width=4pt]}] (0, 0) -- (0, 4.2);
\filldraw [gray] (3.5*.9, 3*.9) circle  (6pt) node [text=white] {\LARGE\textbf{$+$}};
\filldraw [OliveGreen!85] (2.5*.9, 3.5*.9) circle  (6pt) node [text=white] {\LARGE\textbf{$-$}};
\filldraw [OliveGreen!85] (3, 2)  circle  (6pt) node [text=white] {\LARGE\textbf{$-$}};
\filldraw [OliveGreen!85] (2, 2.5)  circle  (6pt) node [text=white]{\LARGE\textbf{$+$}};
\filldraw [gray] (2.7, 1.4) circle  (6pt) node [text=white] {\LARGE\textbf{$+$}};
\filldraw [gray] (2, 1.8) circle  (6pt) node [text=white] {\LARGE\textbf{$-$}};
\filldraw [gray] (1.6, 2.2) circle  (6pt) node [text=white] {\LARGE\textbf{$-$}};
\filldraw [gray] (2, 1.2) circle  (6pt) node [text=white] {\LARGE\textbf{$-$}};
\filldraw [gray] (1.2, 1.6) circle  (6pt) node [text=white] {\LARGE\textbf{$-$}};
\filldraw [red!85] (1.6, .8)  circle  (6pt) node [text=white] {\LARGE\textbf{$-$}};
\filldraw [red!85](0.5, 1.3)  circle  (6pt) node [text=white] {\LARGE\textbf{$+$}};
\filldraw [red!85] (0.8, 0.5) circle  (6pt) node [text=white] {\LARGE\textbf{$+$}};
\draw (0.8, 2.5) node {$DIS(V_i)$};
\draw[-{Triangle[width=18pt,length=8pt]}, line width=10pt](4,2) -- (5, 2);
\draw (4.5, 3.55)node {\small Query};
\draw (4.5, 3.2)node {\small $DIS(V_i)$ };
\draw (4.5, 2.85) node {\small human's};
\draw (4.5, 2.5) node {\small prediction};
\draw (0, -0.5) node [anchor = west] {\textbf{Step 1:} Identify points to query};
\draw (0, -1) node [anchor = west] {for human's prediction};

\end{tikzpicture}
\begin{tikzpicture}
\draw [ -{Triangle[width=4pt]}] (0, 0) --(4.2, 0);
\draw [ -{Triangle[width=4pt]}] (0, 0) -- (0, 4.2);
\filldraw [gray] (3.5*.9, 3*.9) circle  (6pt) node [text=white] {\LARGE\textbf{$+$}};
\filldraw [OliveGreen!85] (2.5*.9, 3.5*.9) circle  (6pt) node [text=white] {\LARGE\textbf{$-$}};
\filldraw [OliveGreen!85] (3, 2)  circle  (6pt) node [text=white] {\LARGE\textbf{$-$}};
\filldraw [OliveGreen!85] (2, 2.5)  circle  (6pt) node [text=white]{\LARGE\textbf{$+$}};
\filldraw [OliveGreen!85] (2.7, 1.4) circle  (6pt) node [text=white] {\LARGE\textbf{$+$}};
\filldraw [gray] (2, 1.8) circle  (6pt) node [text=white] {\LARGE\textbf{$-$}};
\filldraw [OliveGreen!85] (1.6, 2.2) circle  (6pt) node [text=white] {\LARGE\textbf{$-$}};
\filldraw [gray] (2, 1.2) circle  (6pt) node [text=white] {\LARGE\textbf{$-$}};
\filldraw [red!85] (1.2, 1.6) circle  (6pt) node [text=white] {\LARGE\textbf{$-$}};
\filldraw [red!85] (1.6, .8)  circle  (6pt) node [text=white] {\LARGE\textbf{$-$}};
\filldraw [red!85](0.5, 1.3)  circle  (6pt) node [text=white] {\LARGE\textbf{$+$}};
\filldraw [red!85] (0.8, 0.5) circle  (6pt) node [text=white] {\LARGE\textbf{$+$}};
\draw (0.5, 3) node {$d(\boldsymbol{x})$};
\draw [very thick] (0.2,2.8) -- (4., 0.15);
\draw [  -{Triangle[width=4pt]}] (3.24, 0.68) -- (3.14,0.52);
\draw [-{Triangle[width=4pt]}] (0.86-0.05, 2.38) -- (0.76-0.05,2.22);
\draw[-{Triangle[width=18pt,length=8pt]}, line width=10pt](4,2) -- (5, 2);
\draw (4.5, 3.2)node {\small Find};
\draw (4.5, 2.85)node {\small consistent };
\draw (4.5, 2.5) node {\small $(h, r)$};
\draw (0, -0.5) node [anchor = west] {\textbf{Step 2:} Query for human's predictions};
\draw (0, -1) node [anchor = west] {and learn their error boundary $d$  };\end{tikzpicture}
\begin{tikzpicture}
\draw [ -{Triangle[width=4pt]}] (0, 0) --(4.2, 0);
\draw [ -{Triangle[width=4pt]}] (0, 0) -- (0, 4.2);
\filldraw [gray] (3.5*.9, 3*.9) circle  (6pt) node [text=white] {\LARGE\textbf{$+$}};
\filldraw [OliveGreen!85] (2.5*.9, 3.5*.9) circle  (6pt) node [text=white] {\LARGE\textbf{$-$}};
\filldraw [OliveGreen!85] (3, 2)  circle  (6pt) node [text=white] {\LARGE\textbf{$-$}};
\filldraw [OliveGreen!85] (2, 2.5)  circle  (6pt) node [text=white]{\LARGE\textbf{$+$}};
\filldraw [OliveGreen!85] (2.7, 1.4) circle  (6pt) node [text=white] {\LARGE\textbf{$+$}};
\filldraw [gray] (2, 1.8) circle  (6pt) node [text=white] {\LARGE\textbf{$-$}};
\filldraw [OliveGreen!85] (1.6, 2.2) circle  (6pt) node [text=white] {\LARGE\textbf{$-$}};
\filldraw [gray] (2, 1.2) circle  (6pt) node [text=white] {\LARGE\textbf{$-$}};
\filldraw [red!85] (1.2, 1.6) circle  (6pt) node [text=white] {\LARGE\textbf{$-$}};
\filldraw [red!85] (1.6, .8)  circle  (6pt) node [text=white] {\LARGE\textbf{$-$}};
\filldraw [red!85](0.5, 1.3)  circle  (6pt) node [text=white] {\LARGE\textbf{$+$}};
\filldraw [red!85] (0.8, 0.5) circle  (6pt) node [text=white] {\LARGE\textbf{$+$}};
\draw (3.5, 1) node {$r(\boldsymbol{x})$};
\draw (0.8, 3.8) node {$h(\boldsymbol{x})$};
\draw [draw=orange,  -{Triangle[width=4pt]}] (3.24, 0.68) -- (3.34,0.86);
\draw [draw=orange, -{Triangle[width=4pt]}] (0.86-0.05, 2.38) -- (0.96-0.05,2.56);
\draw [very thick, draw=orange] (0.2,2.8) -- (4., 0.15);
\draw [very thick, draw= OliveGreen!50!black!50] (0.2, 4) -- (1.3, 0.);
\draw [draw=OliveGreen!50!black!50,  -{Triangle[width=4pt]}] (1.08, 0.8) -- (0.88,0.75);
\draw [draw=OliveGreen!50!black!50,  -{Triangle[width=4pt]}] (0.42, 3.2) -- (0.22,3.15);
\draw (0, -0.5) node [anchor = west] {\textbf{Step 3:} Learn classifier-rejector};
\draw (0, -1) node [anchor = west] {pair given human's predictions };
\end{tikzpicture}

\begin{tikzpicture}[scale=0.75, every node/.style={scale=0.7}]

\draw  (-0.3, 1.2) -- (2.3, 1.2) -- (2.3, 3.9) -- (-.3, 3.9) -- (-.3, 1.2);
\filldraw [red!85] (0, 3.5) circle  (6pt) node [text=black, anchor=west] { \hspace{0.2cm}$\textrm{Human} \neq Y$};
\filldraw [OliveGreen!85] (0, 3.) circle  (6pt) node [text=black, anchor=west] {\hspace{0.2cm}$\textrm{Human} = Y$};
\filldraw [gray] (0, 2.5) circle  (6pt) node [text=black, anchor=west] {\hspace{0.2cm}Unlabeled};
\draw (0, 1.5+0.5) node {\LARGE $+$} node [anchor=west] {\hspace{0.2cm}$Y=1$};
\draw (0, 1.+0.5) node {\LARGE $-$} node [anchor=west] {\hspace{0.2cm}$Y=0$};
\draw (0, -2.3) node {};
\end{tikzpicture}
}
\caption{Illustration of our active learning algorithm Disagreement on Disagreements (DoD) \eqref{alg:disagreement}. At each round, we compute the disagreement set for our predictors of the human label disagreement, we then query the human for their prediction on these points. After we learn the expert error boundary, we then learn a consistent classifier-rejector pair.  }
\end{figure*}

\begin{theorem} \label{thm: gen_cross}
Let $K$ denote the number of classes, and let $\Fcal$ be a hypothesis class of functions $f_i: \Xcal\to \Rbb$ with bounded infinity norm $\|f_i\|_{\infty}<C$. 
Given $\hat{\fv} \in \Fcal^{k+1}$ the empirical minimizer of the surrogate loss $\tilde{L}_{CE}$, then
we have  with probability at least $1-\delta$, we have
\ea{ \nonumber
\psi\big(L_{\mathrm{def}}^{0 - 1}(h_{\hat{\fv}}, r_{\hat{\fv}})-\min_{\fv} L_{\mathrm{def}}^{0 - 1}(h_{\fv},r_{\fv}) \big)&\leq 2(K+1) \Rad_{n}\big(\Fcal\big)+ \sqrt{\frac{(8C-4\log (K+1))\log 2/\delta}{n}}\nonumber\\&\quad+2(K+2)\min\{\Pr(M\neq Y), \Rad_{n\Pr(M\neq Y)/2}\big(\Fcal\big)\}\nonumber\\&\quad \nonumber\\&\quad+C\Pr(M\neq Y)(k+2)e^{-n\Pr^2(M\neq Y)/2} + e_{\phi-\mathrm{appr}},
}
 where $e_{\phi-\mathrm{appr}}$ is the approximation error for the $\phi$-surrogate, which is defined as
\ea{
e_{\phi-\mathrm{appr}} = \min_{\fv\in \Fcal^{k+1}} \tilde{L}_{CE}(\fv) - \min_{\fv} \tilde{L}_{CE}(\fv).
}

\end{theorem}
Proof of the theorem can be found in Appendix \ref{app: gen_CE}.
Comparing the sample complexity estimate for minimizing the surrogate to that of minimizing the 0-1 loss as computed by \cite{mozannar2020consistent}, we find that we pay an additional penalty for the complexity of the hypothesis class in addition to the higher sample complexity that scales with $O(\frac{\log \ep}{\ep^{4}})$ due to the calibration function. To compensate for such increase in sample complexity, in the next section we seek to design active learning schemes that reduce the required amount of human labels for learning.  

\section{Active Learning for Expert Predictions}\label{sec: active}

\subsection{Theoretical Understanding}\label{subsec:theory_active_learning}
In Section \ref{sec: staged_learning}, we assumed that we have a randomly selected subset of data that is labeled by the human expert. In a practical setting, we may assume that we have the ability to choose which points we would like the human expert to predict on. For example, when we deploy an X-ray diagnostic algorithm to a new hospital, we can interact with each radiologist for a few rounds to build a tailored classifier-rejector pairs according to their individual abilities.

Therefore, we assume that we have access to the distribution of instances $\xv$ and their labels and we could query for the expert's prediction on each instance. The goal is to query the human expert on as few instances as possible while being able to learn an approximately optimal classifier-rejector pair.
To make progress in theoretical understanding, we assume that we can achieve zero loss with an optimal classifier-rejector pair:
\begin{Assum}[Realizability]\label{assump: real}
We assume that the data is realizable by our joint class $(\mathcal{H},\mathcal{R})$: there exists $h^*,r^* \in \mathcal{H} \times \mathcal{R}$ that have zero error $L(h^*,r^*)=0$.
\end{Assum}
In this section, the algorithms we develop apply to the multiclass setting but we restrict the theoretical analysis to binary labels. 
The fundamental algorithm in active learning in the realizable case for classification is the CAL algorithm
\cite{hanneke2014theory}. The algorithm keeps track of a version class of the hypothesis space that is consistent with the data so far and then at each step computes the disagreement set: the points on which there exists two classifiers in the hypothesis class with different predictions, and then picks at random a set of points from this disagreement set. 
We start by initializing our version space by taking advantage of Assumption \ref{assump: real}:
\begin{align}
V_0 = \{ h,r \in \Hcal \times \Rcal:& \forall \xv,   r(\xv)=0 \to  \ h(\xv)=y \}
\end{align}
The above initialization of the version space assumes we know the label of all instances in our support. Alternatively, one could collect at most $O\big(\frac{1}{\ep}(d(\Hcal)\log \frac{1}{\ep}+\log\frac{1}{\delta})\big)$ labels of instances and that would be sufficient to test for realizability of our classifier with error $\ep$ (see Lemma 3.2 of \cite{hanneke2014theory}).

The main difference with  active learning for classification is that we are not able to compute the disagreement set for the overall prediction of the deferral system as it requires knowing the expert predictions. However, we know that a necessary condition for disagreement is that there exists a feasible pair of classifiers-rejectors where the rejectors disagree. 
Suppose $(h_1,r_1)$ and $(h_2,r_2)$ are in our current version space. These two pairs can only disagree when on an instance $\xv$: $r_1(\xv) \neq r_2(\xv)$, since otherwise when both defer, the expert makes the same prediction, and when both do not defer, both classify the label correctly by the realizability assumption.
Thus, we define the disagreement set in terms of only the rejectors that are in the version space at each round $j$:
\begin{align}
    DIS(V_{j-1}) = \{ x \in \Xcal\mid\, \exists (h_1,r_1),(h_2,r_2) \in V_{j-1}  s.t. \ r_1(x) \neq r_2(x) \}
\end{align}
Then we ask for the labels of $k$ instances in $DIS(V_{j-1})$ to form $\Scal_j=\{(\xv_i, y_i, m_i)\,:\,\xv_i\in DIS(V_j)\}$ and we update the version space as
\ea{
V_j = \{(h, r)\in V_{j-1}\mid \, \forall (\xv, y, m)\in \Scal_j,  \text{if} \ r(\xv)=1\rightarrow y=m\}
}
Now, we prove that the above rejector-disagreement algorithm will converge if the optimal unique classifier-rejector pair is unique: 
\begin{proposition}\label{prop: active_1}
Assume that there exists a unique pair $(h^*,r^*) \in \mathcal{H} \times \mathcal{R}$ that have zero error $L(h^*,r^*)=0$. Let $\Theta$ be defined as:
\begin{equation}
\Theta=\sup_{t>0}\frac{\Pr\left(X\in DIS\left(B\left((h^*, r^*), t\right)\right)\right)}{t}
\end{equation}
where $B\big((h, r), t\big)= \{(h', r')\in \Hcal\times \Rcal: \Pr(r(X)M+(1-r(X))h(X)\neq r'(x)M+(1-r'(X))h'(X))\leq t\}$.

Then, running the rejector-disagreement algorithm with $k =O\big( \Theta^2((d(\Hcal) + d(\Rcal)\log \Theta+ \log \tfrac{1}{\delta}+\log \log\tfrac{1}{\ep})\big) $ for $\log(1/\ep)$ iterations will return classifier-rejector with $\epsilon$ error and with probability at least $1- \delta$.
\end{proposition}
Proof of the proposition can be found in Appendix \ref{app: active_1}.
\subsection{Disagreement on Disagreements}

If we remove the uniqueness assumption  for the rejector-disagreement algorithm in the previous subsection, we show in Appendix \ref{app: counter_CAL}  with an example that the algorithm no longer converges as $DIS(V)$ can remain constant. We expect that the uniqueness assumption may not hold in practice, so we now hope to design algorithms that do not require it. Instead, we now make a new assumption that we can learn the error boundary of the human expert via a function $f \in \mathcal{D}$, that is given any sample $(x,y,m)$ we have  $f(x) = \bI_{y \neq m}$. This assumption is identical to those made in active learning for cost-sensitive classification \cite{krishnamurthy2017active}. This assumption is formalized as follows:

\begin{Assum}\label{assump: dis}
We assume that there exists $f^*\in \Dcal$ such that $\Pr(\bI_{M\neq Y}\neq f(X))=0$.
\end{Assum}

Our new active learning will now seek to directly take advantage of Assumption \ref{assump: dis}. The algorithm consists of two stages: the first stage runs a standard active learning algorithm, namely CAL, on the hypothesis space $\Dcal$ to learn the expert disagreement with the label with error at most $\ep$. In the second stage, we label our data with the predictor $\hat{f}$ that is the output of the first stage, and then learn a classifier-rejector pair from this pseudo-labeled data. Key to this two stage process, is to show that the error from the first stage is not too amplified by the second stage. The algorithm is named Disagreement on Disagreements (DoD)  and is described in detail in Algorithm box \ref{alg:disagreement}.

\begin{algorithm}[h]
	\DontPrintSemicolon 
	\SetAlgoLined
	\textbf{Input}: parameter $n_u$, $T$, $k$, class $\Dcal$, $\Hcal$, and $\Rcal$ \\
	1. $V \gets \Dcal$\\
2. \For{$i\in \{1, \ldots, T\}$}{
Sample from $\mu_{X}$ until you have $k$ samples $\{\xv_i\}_{i=1}^{k}$ within $DIS_2(V)$ \\
Query for $\{(y_i, m_i)\}_{i=1}^{k}$ for the samples $\{\xv_i\}_{i=1}^{k}$\\
Update $V\leftarrow \{d\in V:\, \forall j\, d(\xv_j)=\bI_{m_j\neq y_j}\}$
	}
4. Collect $n_u$ samples $\{(\xv_i', y_i')\}_{i=1}^{n_u}$ from $\mu_{XY}$\\
	\textbf{Return:}  $(h, r) \in \Hcal\times \Rcal$ such that $\sum_{(\xv_i', y_i')}\bI_{h(\xv_i')\neq y_i'}\big(1-r(\xv_i')\big)+r(\xv_i')d(\xv_i')=0$, for some $d\in V$ \\
	\caption{Active Learning algorithm DoD (Disagreement on disagreements)}
	\label{alg:disagreement}
\end{algorithm}

In the following we prove a label complexity guarantee for  Algorithm \ref{alg:disagreement}.

\begin{theorem}\label{thm: Dod}
Let us define $\Theta_2$ as
\ea{
\Theta_2 = \sup_{t>0}\frac{\Pr\left(X\in DIS_2\left(B_2(f^*, t)\right)\right)}{r},
}
where $B_2(f, t)=\{f'\in \Dcal\mid \, \Pr(f'(X)\neq f(X))\leq t\}$, and $DIS_2(V)=\{\xv\in \Xcal\mid\, \,\exists f_1, f_2\in V,\,  f_1(\xv)\neq f_2(\xv)\}$.

Assume we have $\Hcal, \Rcal, \Dcal$ that satisfy assumption \ref{assump: real} and \ref{assump: dis}. Then for $n_u=O(\frac{\log 1/\delta+\max\{d(\Hcal), d(\Rcal)\}\log 1/\ep}{\ep^{2}})$, and $k=d(\Dcal)\Theta_2\log(\tfrac{\Theta_2}{\delta}\log (\tfrac{1}{\ep}))$, then Algorithm \ref{alg:disagreement} takes $T=O(\log \frac{1}{\ep})$ iterations to output a solution with $\ep$-upper-bound on deferral loss with probability at least $1-\delta$. As a result, the sample complexity of labeled data $n_l$ is $O\big(d(\Dcal)\Theta_2\log(\tfrac{\Theta_2}{\delta}\log (\tfrac{1}{\ep}))\log (\tfrac{1}{\ep})\big)$.
\end{theorem}

Proof of the proposition can be found in Appendix \ref{app: active}.
Recall that in Proposition \ref{prop: gen_labeled_unlabeled}, where the labeled data was chosen at random, the sample complexity $n_u$ is in order $O(\frac{1}{\ep^2}\log \frac{1}{\ep})$. As we see in Theorem \ref{thm: Dod}, the proposed active learning algorithm reduces sample complexity to $O(\log \frac{1}{\ep})$, with the caveat that realizability is assumed for active learning.
Further, note that for this algorithm, in contrast to previous subsection, the uniqueness of the consistent pair $(h, r)$ is not needed anymore. However, this algorithm ignores the classifier and rejector classes when querying for points, which  makes the sample complexity $n_l$  dependent only on the complexity of $\Dcal$ instead of $\Hcal, \Rcal$.  In the next section, we try to understand how to use surrogate loss functions to practically optimize for our classifier-rejector pair.

\section{Experimental Illustration}

Code for our experiments is found in \url{https://github.com/clinicalml/active_learn_to_defer}.
\paragraph{Dataset.} We use the CIFAR-10 image
classification dataset \cite{krizhevsky2009learning} consisting of $32 \times32$ color
images drawn from 10  classes. 
 We consider the human expert models considered in \cite{mozannar2020consistent}:  if the image is in the first 5 classes the human expert is perfect, otherwise the expert predicts randomly. Further experimental details are in Appendix \ref{app: experiments}.  
 
\paragraph{Model and Optimization.} We parameterize the classifier and rejector by a convolutional neural network consisting of two convolutional layers followed by two feedforward layers. For staged learning, we train the classifer on  the training data optimizing for performance on a validation set, and for the rejector we train a network to predict the expert error and defer at test time by comparing the confidence of the classifier and the expert as in \cite{raghu2019algorithmic}. For joint learning, we use the loss $L_{CE}^\alpha$, a simple extension of the loss \eqref{eqn:deferral_loss} in \cite{mozannar2020consistent}, optimizing the parameter $\alpha$ on a validation set.  

\paragraph{Model Complexity Gap.} In Figure \ref{fig:model_complex}, we plot the difference of accuracy between joint learning and staged learning  as we increase the complexity of the classifier class by increasing the filter size of the convolutional layers and the number of units in the feedforward layers. Model complexity is captured by the number of parameters in the classifier which serves only as a rough proxy of the VC dimension that varies in the same direction. The difference is decreasing as predicted by Theorem \ref{thm: vc} as we increase the classifier class complexity as we fix the complexity of the rejector. 
\begin{figure}[H]
    \centering
    \includegraphics[scale=.7]{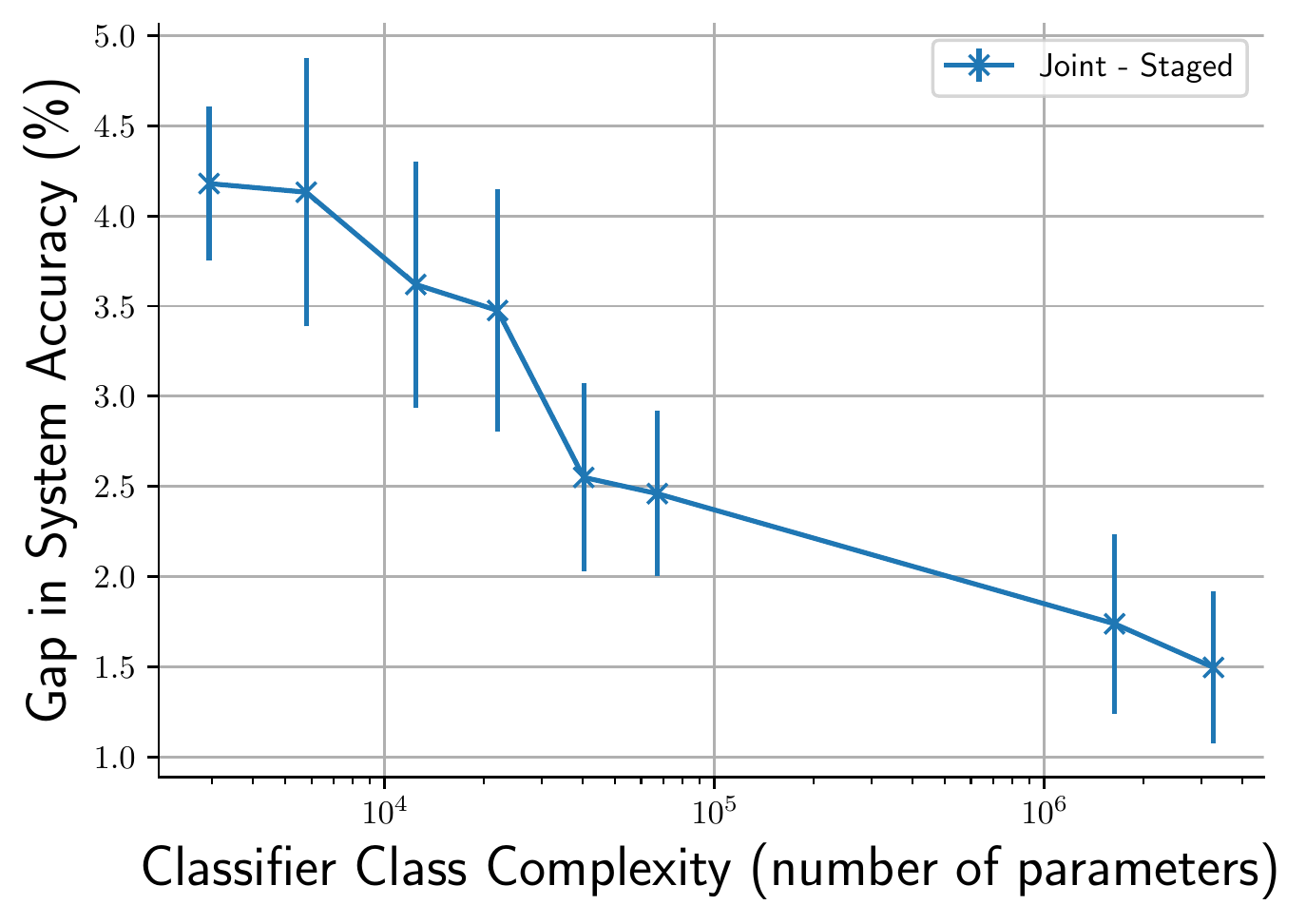}
    \caption{Difference of accuracy between joint learning and staged learning of the classifier-rejector pair (y-axis is log scale of number of parameters). }
    \label{fig:model_complex}
\end{figure}
\paragraph{Data Trade-Offs.} In Figure \ref{fig:sample_complex}, we plot the of accuracy between joint learning and staged learning when only a subset of the data is labeled by the expert as in Section \ref{subsec:data_tradeofs}.  We plot the average difference across 10 trials and error bars denote standard deviation. We only plot the performance of joint learning when initialized first on the unlabeled data to predict the target and then trained on the labeled expert data to defer, we denote this approach as 'Joint-SemiSupervised'.  For staged learning, the classifier is trained on all of the data $S_l \cup S_u$, while for joint learning we only train on $S_l$. We can see that when there is more unlabeled data than labeled, staged learning outperforms joint learning in accordance with Proposition \ref{prop: gen_labeled_unlabeled}. The heuristic method 'Joint-SemiSupervised' improves on the sample complexity of 'Joint' but still lags behind the Staged approach in low data regimes.
\begin{figure}[H]
    \centering
    \includegraphics[scale=.7]{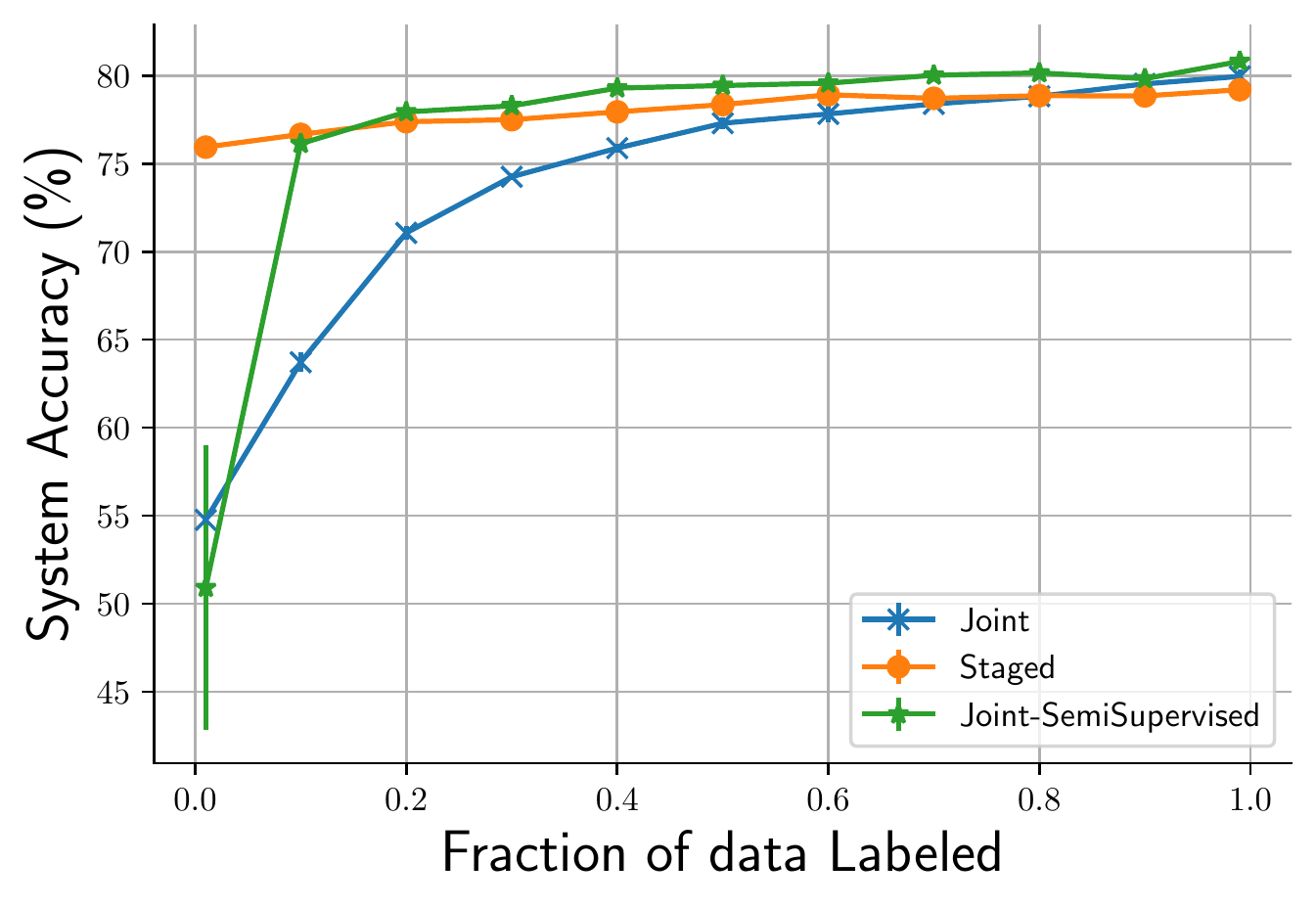}
    \caption{Performance of joint learning and staged learning as we increase the ratio  of the data labeled by the expert $\frac{n_l}{n_u + n_l} $.}
    \label{fig:sample_complex}
\end{figure}
\paragraph{DoD algorithm.}
In Figure \ref{fig:DoD}, we plot corresponding errors of the DoD algorithm and we compare them to the staged and joint learning. The features $\xv$ of the synthetic data in here is generated from a uniform distribution on interval $[0, 1]$, and the labels $y$ are equal to $1$ where $\xv>0.3$ (full-information region) and are equal to random outcomes of a $Bernoulli(0.5)$ distribution otherwise (no-information region). The human's decision is inaccurate ($M\neq Y$) for $X>0.3$ and accurate ($M=Y$) otherwise. We further assume each hypothesis class of rejectors and classifiers be 100 samples of half-spaces over the interval $[0, 1]$. The error plotted in Figure \ref{fig:DoD} is an average of 1000 random generations of training data. The test set is formed by $N_{test}=1000$ samples that are generated from the same distribution as training data. Here, we note that the number of unlabeled data in staged learning is set $N_{u}=100$. The result of this experiment shows that in the DoD algorithm, one needs less number of samples that are labeled by human to reach a similar level of error.
\begin{figure}[H]
    \centering
    \includegraphics[scale=.7]{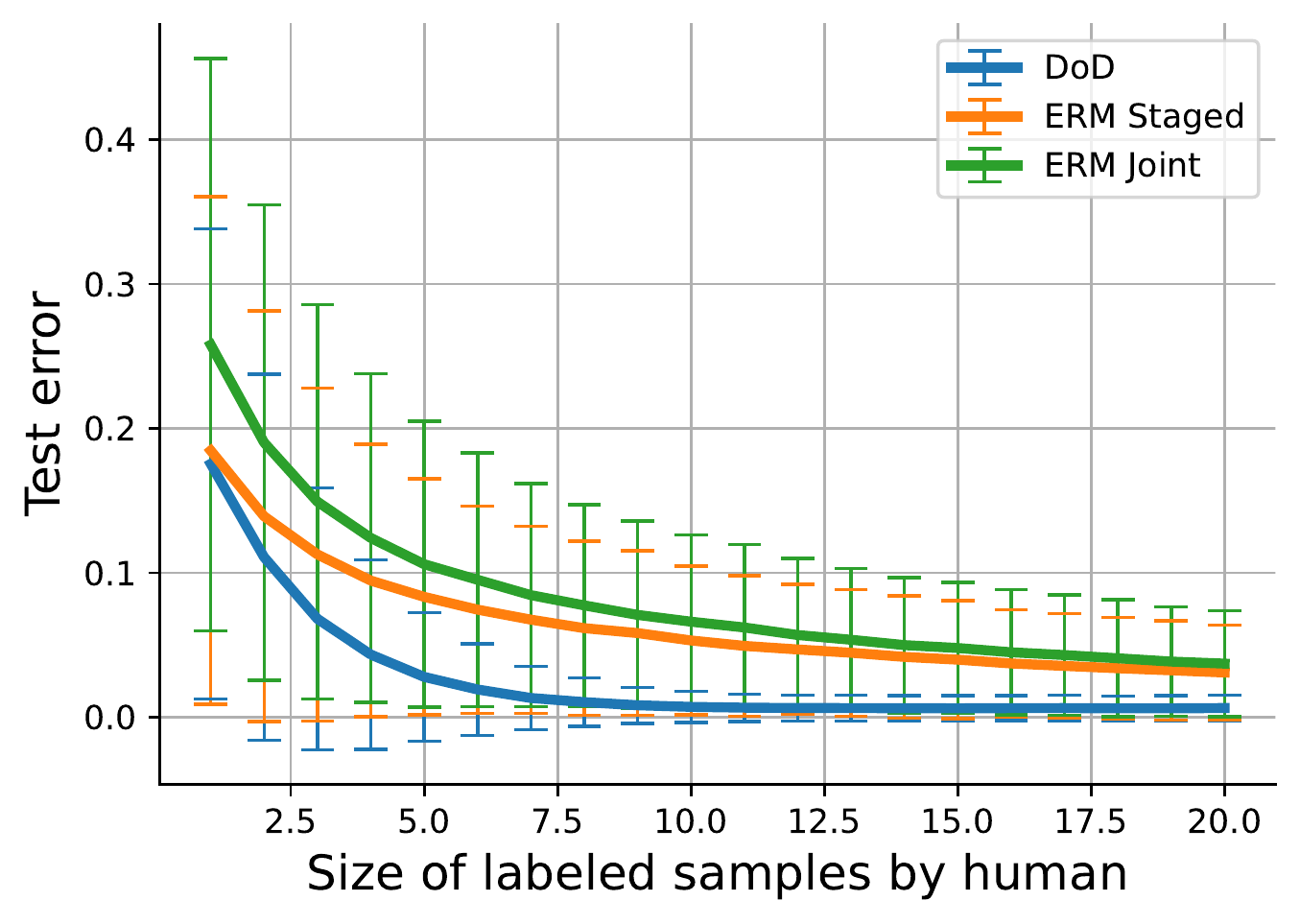}
    \caption{Error of the DoD algorithm compared to staged and joint learning for increasing number of training data that are labeled by human.}
    \label{fig:DoD}
\end{figure}
\section{Discussion}
In this work, we provided  novel theoretical analysis of learning in the expert deferral setting. We first analyzed the gap in performance between jointly learning a classifier and rejector, and a staged learning approach. While our theorem on the gap is a worst-case statement, an experimental illustration on CIFAR-10 indicates a more general trend. Further analysis could explicitly compute the gap for certain hypothesis classes of interest. We further analyzed a popular approach to jointly learning to defer, namely consistent surrogate loss functions. To that end, we proposed a novel family of surrogates that generalize prior work and give a criteria, namely the surrogate excess-risk bound for evaluating surrogates. Future work will try to instantiate  members of this family that minimize the excess-risk bound and provide improved empirical performances.
Driven by the limited availability of human data, we sought to design active learning schemes that requires a minimal amount of labeled data for learning a classifier-rejector pair. While our results hold for the realizable setting, we believe it is feasible to generalize to the agnostic setting. Future work will also build and test practical active learning algorithms inspired by our theoretical analysis.

\section*{Acknowledgements}
M.A. Charusaie thanks the International Max Planck Research
School for Intelligent Systems (IMPRS-IS) for the support and funding of this project. The authors would like to thank Nico Gürtler, Jack Brady, Michael Muehlebach, and Hunter Lang for helpful feedback and other members of the Clinical ML group.


\bibliographystyle{alpha}
\bibliography{ref}

\newcommand{\etalchar}[1]{$^{#1}$}
\begin{thebibliography}{GSTDA{\etalchar{+}}21}

\bibitem[AGS20]{acar2020budget}
Durmus Alp~Emre Acar, Aditya Gangrade, and Venkatesh Saligrama.
\newblock Budget learning via bracketing.
\newblock In {\em International Conference on Artificial Intelligence and
  Statistics}, pages 4109--4119. PMLR, 2020.

\bibitem[BBH{\etalchar{+}}20]{beede2020human}
Emma Beede, Elizabeth Baylor, Fred Hersch, Anna Iurchenko, Lauren Wilcox,
  Paisan Ruamviboonsuk, and Laura~M Vardoulakis.
\newblock A human-centered evaluation of a deep learning system deployed in
  clinics for the detection of diabetic retinopathy.
\newblock In {\em Proceedings of the 2020 CHI Conference on Human Factors in
  Computing Systems}, pages 1--12, 2020.

\bibitem[BJM06]{bartlett2006convexity}
Peter~L Bartlett, Michael~I Jordan, and Jon~D McAuliffe.
\newblock Convexity, classification, and risk bounds.
\newblock {\em Journal of the American Statistical Association},
  101(473):138--156, 2006.

\bibitem[BNK{\etalchar{+}}19]{bansal2019updates}
Gagan Bansal, Besmira Nushi, Ece Kamar, Daniel~S Weld, Walter~S Lasecki, and
  Eric Horvitz.
\newblock Updates in human-ai teams: Understanding and addressing the
  performance/compatibility tradeoff.
\newblock In {\em Proceedings of the AAAI Conference on Artificial
  Intelligence}, volume~33, pages 2429--2437, 2019.

\bibitem[BW08]{bartlett2008classification}
Peter~L Bartlett and Marten~H Wegkamp.
\newblock Classification with a reject option using a hinge loss.
\newblock {\em Journal of Machine Learning Research}, 9(Aug):1823--1840, 2008.

\bibitem[CCZS21]{charoenphakdee2021classification}
Nontawat Charoenphakdee, Zhenghang Cui, Yivan Zhang, and Masashi Sugiyama.
\newblock Classification with rejection based on cost-sensitive classification.
\newblock In {\em International Conference on Machine Learning}, pages
  1507--1517. PMLR, 2021.

\bibitem[CDM16]{cortes2016learning}
Corinna Cortes, Giulia DeSalvo, and Mehryar Mohri.
\newblock Learning with rejection.
\newblock In {\em International Conference on Algorithmic Learning Theory},
  pages 67--82. Springer, 2016.

\bibitem[Cho70]{chow1970optimum}
C~Chow.
\newblock On optimum recognition error and reject tradeoff.
\newblock {\em IEEE Transactions on information theory}, 16(1):41--46, 1970.

\bibitem[DWMW17]{davidson2017automated}
Thomas Davidson, Dana Warmsley, Michael Macy, and Ingmar Weber.
\newblock Automated hate speech detection and the problem of offensive
  language.
\newblock In {\em Eleventh international aaai conference on web and social
  media}, 2017.

\bibitem[EYW10]{el2010foundations}
Ran El-Yaniv and Yair Wiener.
\newblock On the foundations of noise-free selective classification.
\newblock {\em Journal of Machine Learning Research}, 11(May):1605--1641, 2010.

\bibitem[GEY17]{geifman2017selective}
Yonatan Geifman and Ran El-Yaniv.
\newblock Selective classification for deep neural networks.
\newblock In {\em Advances in neural information processing systems}, pages
  4878--4887, 2017.

\bibitem[Gil20]{gillespie2020content}
Tarleton Gillespie.
\newblock Content moderation, ai, and the question of scale.
\newblock {\em Big Data \& Society}, 7(2):2053951720943234, 2020.

\bibitem[GKS21]{gangrade2021selective}
Aditya Gangrade, Anil Kag, and Venkatesh Saligrama.
\newblock Selective classification via one-sided prediction.
\newblock In {\em International Conference on Artificial Intelligence and
  Statistics}, pages 2179--2187. PMLR, 2021.

\bibitem[GSTDA{\etalchar{+}}21]{gao2021human}
Ruijiang Gao, Maytal Saar-Tsechansky, Maria De-Arteaga, Ligong Han, Min~Kyung
  Lee, and Matthew Lease.
\newblock Human-ai collaboration with bandit feedback.
\newblock {\em arXiv preprint arXiv:2105.10614}, 2021.

\bibitem[Han14]{hanneke2014theory}
Steve Hanneke.
\newblock Theory of active learning.
\newblock {\em Foundations and Trends in Machine Learning}, 7(2-3), 2014.

\bibitem[JPDV21]{joshi2021pre}
Shalmali Joshi, Sonali Parbhoo, and Finale Doshi-Velez.
\newblock Pre-emptive learning-to-defer for sequential medical decision-making
  under uncertainty.
\newblock {\em arXiv preprint arXiv:2109.06312}, 2021.

\bibitem[KAH{\etalchar{+}}17]{krishnamurthy2017active}
Akshay Krishnamurthy, Alekh Agarwal, Tzu-Kuo Huang, Hal Daum{\'e}~III, and John
  Langford.
\newblock Active learning for cost-sensitive classification.
\newblock In {\em International Conference on Machine Learning}, pages
  1915--1924. PMLR, 2017.

\bibitem[KH{\etalchar{+}}09]{krizhevsky2009learning}
Alex Krizhevsky, Geoffrey Hinton, et~al.
\newblock Learning multiple layers of features from tiny images.
\newblock {\em Citeseer}, 2009.

\bibitem[KHH12]{kamar2012combining}
Ece Kamar, Severin Hacker, and Eric Horvitz.
\newblock Combining human and machine intelligence in large-scale
  crowdsourcing.
\newblock In {\em AAMAS}, volume~12, pages 467--474, 2012.

\bibitem[KLK21]{keswani2021towards}
Vijay Keswani, Matthew Lease, and Krishnaram Kenthapadi.
\newblock Towards unbiased and accurate deferral to multiple experts.
\newblock {\em arXiv preprint arXiv:2102.13004}, 2021.

\bibitem[KLL{\etalchar{+}}18]{kleinberg2018human}
Jon Kleinberg, Himabindu Lakkaraju, Jure Leskovec, Jens Ludwig, and Sendhil
  Mullainathan.
\newblock Human decisions and machine predictions.
\newblock {\em The quarterly journal of economics}, 133(1):237--293, 2018.

\bibitem[KSS21]{kerrigan2021combining}
Gavin Kerrigan, Padhraic Smyth, and Mark Steyvers.
\newblock Combining human predictions with model probabilities via confusion
  matrices and calibration.
\newblock {\em Advances in Neural Information Processing Systems}, 34, 2021.

\bibitem[LGB21]{liu2021incorporating}
Jessie Liu, Blanca Gallego, and Sebastiano Barbieri.
\newblock Incorporating uncertainty in learning to defer algorithms for safe
  computer-aided diagnosis.
\newblock {\em arXiv preprint arXiv:2108.07392}, 2021.

\bibitem[LH17]{loshchilov2017decoupled}
Ilya Loshchilov and Frank Hutter.
\newblock Decoupled weight decay regularization.
\newblock {\em arXiv preprint arXiv:1711.05101}, 2017.

\bibitem[LLT20]{lai2020chicago}
Vivian Lai, Han Liu, and Chenhao Tan.
\newblock " why is' chicago'deceptive?" towards building model-driven tutorials
  for humans.
\newblock In {\em Proceedings of the 2020 CHI Conference on Human Factors in
  Computing Systems}, pages 1--13, 2020.

\bibitem[LT91]{ledoux1991probability}
Michel Ledoux and Michel Talagrand.
\newblock {\em Probability in Banach Spaces: isoperimetry and processes},
  volume~23.
\newblock Springer Science \& Business Media, 1991.

\bibitem[MPZ18]{madras2018predict}
David Madras, Toni Pitassi, and Richard Zemel.
\newblock Predict responsibly: Improving fairness and accuracy by learning to
  defer.
\newblock In {\em Advances in Neural Information Processing Systems}, pages
  6150--6160, 2018.

\bibitem[MRT18]{mohri2018foundations}
Mehryar Mohri, Afshin Rostamizadeh, and Ameet Talwalkar.
\newblock {\em Foundations of machine learning}.
\newblock MIT press, 2018.

\bibitem[MS20]{mozannar2020consistent}
Hussein Mozannar and David Sontag.
\newblock Consistent estimators for learning to defer to an expert.
\newblock In {\em International Conference on Machine Learning}, pages
  7076--7087. PMLR, 2020.

\bibitem[MSS22]{mozannar2021teaching}
Hussein Mozannar, Arvind Satyanarayan, and David Sontag.
\newblock Teaching humans when to defer to a classifier via exemplars.
\newblock In {\em Proceedings of the Thirty-Sixth AAAI Conference on Artificial
  Intelligence (AAAI)}, 2022.

\bibitem[NVBR19]{nowak2019general}
Alex Nowak-Vila, Francis Bach, and Alessandro Rudi.
\newblock A general theory for structured prediction with smooth convex
  surrogates.
\newblock {\em arXiv preprint arXiv:1902.01958}, 2019.

\bibitem[ODGR21]{okati2021differentiable}
Nastaran Okati, Abir De, and Manuel Gomez-Rodriguez.
\newblock Differentiable learning under triage.
\newblock {\em arXiv preprint arXiv:2103.08902}, 2021.

\bibitem[PZP{\etalchar{+}}21]{pradier2021preferential}
Melanie~F Pradier, Javier Zazo, Sonali Parbhoo, Roy~H Perlis, Maurizio Zazzi,
  and Finale Doshi-Velez.
\newblock Preferential mixture-of-experts: Interpretable models that rely on
  human expertise as much as possible.
\newblock {\em arXiv preprint arXiv:2101.05360}, 2021.

\bibitem[RBC{\etalchar{+}}19]{raghu2019algorithmic}
Maithra Raghu, Katy Blumer, Greg Corrado, Jon Kleinberg, Ziad Obermeyer, and
  Sendhil Mullainathan.
\newblock The algorithmic automation problem: Prediction, triage, and human
  effort.
\newblock {\em arXiv preprint arXiv:1903.12220}, 2019.

\bibitem[RSG16]{ribeiro2016should}
Marco~Tulio Ribeiro, Sameer Singh, and Carlos Guestrin.
\newblock " why should i trust you?" explaining the predictions of any
  classifier.
\newblock In {\em Proceedings of the 22nd ACM SIGKDD international conference
  on knowledge discovery and data mining}, pages 1135--1144, 2016.

\bibitem[RY21]{raman2021improving}
Naveen Raman and Michael Yee.
\newblock Improving learning-to-defer algorithms through fine-tuning.
\newblock {\em arXiv preprint arXiv:2112.10768}, 2021.

\bibitem[SCMA{\etalchar{+}}17]{su2017interpretable}
Shihan Su, Yuxin Chen, Oisin Mac~Aodha, Pietro Perona, and Yisong Yue.
\newblock Interpretable machine teaching via feature feedback.
\newblock 2017.

\bibitem[SM20]{shah2020online}
Kulin Shah and Naresh Manwani.
\newblock Online active learning of reject option classifiers.
\newblock In {\em Proceedings of the AAAI Conference on Artificial
  Intelligence}, volume~34, pages 5652--5659, 2020.

\bibitem[SSMGR21]{straitouri2021reinforcement}
Eleni Straitouri, Adish Singla, Vahid~Balazadeh Meresht, and Manuel
  Gomez-Rodriguez.
\newblock Reinforcement learning under algorithmic triage.
\newblock {\em arXiv preprint arXiv:2109.11328}, 2021.

\bibitem[TAIK18]{tan2018investigating}
Sarah Tan, Julius Adebayo, Kori Inkpen, and Ece Kamar.
\newblock Investigating human+ machine complementarity for recidivism
  predictions.
\newblock {\em arXiv preprint arXiv:1808.09123}, 2018.

\bibitem[WHK20]{wilder2020learning}
Bryan Wilder, Eric Horvitz, and Ece Kamar.
\newblock Learning to complement humans.
\newblock {\em arXiv preprint arXiv:2005.00582}, 2020.

\bibitem[ZARS21]{zhao2021directing}
Jason Zhao, Monica Agrawal, Pedram Razavi, and David Sontag.
\newblock Directing human attention in event localization for clinical timeline
  creation.
\newblock In {\em Machine Learning for Healthcare Conference}, pages 80--102.
  PMLR, 2021.

\bibitem[ZC15]{zhang2015active}
Chicheng Zhang and Kamalika Chaudhuri.
\newblock Active learning from weak and strong labelers.
\newblock In {\em Advances in Neural Information Processing Systems}, pages
  703--711, 2015.

\bibitem[Zha04]{zhang2004statistical}
Tong Zhang.
\newblock Statistical analysis of some multi-category large margin
  classification methods.
\newblock {\em Journal of Machine Learning Research}, 5(Oct):1225--1251, 2004.

\bibitem[ZSZR18]{zhu2018overview}
Xiaojin Zhu, Adish Singla, Sandra Zilles, and Anna~N Rafferty.
\newblock An overview of machine teaching.
\newblock {\em arXiv preprint arXiv:1801.05927}, 2018.

\end{thebibliography}
\clearpage
\appendix
\onecolumn
\section*{Notations}
We employ the notations $L_{\mathrm{def}}^{\mu_X}$, $L_{\mathrm{def}}^{\mu_X\mu_{Y|X}}$, $L_{\mathrm{def}}^{\mu_{XYM}}$ to indicate $\defl$ and stress on marginal, conditional, and joint probability measures on $X, Y,$ and $M$. We further use $L_{0-1}^{\mu_X\mu_{Y|X}}$ to indicate zero-one loss $L_{0-1}$ and to represent the underlying probability measures on $X$ and $Y$. The cardinality of a set $\Acal$ is indicated by $|\Acal|$.  The notation for the set of numbers from $1$ to $K$ is: $[K] = \{1,\cdots,K\}$.
\section{Proof of Theorem \ref{thm: vc}} \label{app: vc}
We first introduce some useful lemmas as below. 
In Lemma \ref{lem: const_cs}, we show that there exists a pair of hypothesis classes $(\Hcal, \Rcal)$ such that for all non-atomic measures on $\Xcal$ the deferral loss takes a fixed value. In Lemma \ref{lem: bound_diff_discrete}, we use the aforementioned lemma to show that the difference of deferral losses for all two pairs of classifier/rejector $(h_1, r_1)$ and $(h_2, r_2)$ is bounded by the difference of two deferral losses with atomic measures on $\Xcal$. In Lemma \ref{lem: diff_ub}, we upper-bound the difference of two deferral losses for pairs of classifier/rejector that are obtained by staged and joint learning and on hypothesis classes that are defined in Lemma \ref{lem: const_cs}. Such upper-bound is in terms of expected loss of an optimal classifier on a certain hypothesis class. In Lemma \ref{lem: optimal_Hd}, we further calculate the optimal expected loss on such classes. In Lemma \ref{lem: ordered}, we show that on a set of events with size $n$, we could find a subset with size $a$ and probability at most $\frac{a}{n}$. 
Next, we uses these lemmas in the main proof of theorem.
\begin{lemma}\label{lem: const_cs}
For a probability measure $\mu_X$ with no atomic component on $\Xcal$, hypothesis class $\Hcal$ such that for every $h\in \Hcal$, we have $|\{\xv: h(\xv)=1\}|\leq d(\Hcal)$, and hypothesis class $\Rcal$ such that for every $r\in \Rcal$, we have $|\{\xv: r(\xv)=1\}|\leq d(\Rcal)$, for every choice of $(h, r)\in \Hcal\times \Rcal$, the loss $$\defl(h, r)= \Ebb_{X, Y, M} [\bI_{h(X)\neq Y}\bI_{r(X)=0}+\bI_{M\neq Y}\bI_{r(X)=1}],$$ takes a constant value.
\end{lemma}
\begin{proof}
Firstly, we know that
	\ea{
\defl(h, r)&=\Ebb_{X, Y, M}[\bI_{h(X)\neq Y}\bI_{r(X)=0}]+\Ebb_{X, Y, M}[\bI_{M\neq Y}\bI_{r(X)=1}] .\label{eqn: const_l}}

Since probability measure of the set $\{x: r(x)=1\}$ is zero in the absence of atomic components in $\mu_{X}$, one can show that $\Pr(r(X)=1)=0$ (, and equivalently $\Pr(r(X)=0)=1$). This fact together with \eqref{eqn: const_l} concludes that
\ea{
\defl(h, r)= \Ebb_{X, Y} [\bI_{h(X)\neq Y}] . \label{eqn: def_is_nondef}
}
Further, we have 
\ea{
\Ebb_{X, Y} [\bI_{h(X)\neq Y}] &= \Ebb_{X, Y}[\bI_{h(X)\neq Y}|h(X) = 0] \Pr(h(X)=0) +	\Ebb_{X, Y}[\bI_{h(X)\neq Y}|h(X) = 1] \Pr(h(X)=1)\\
&\overset{(a)}{=} \Ebb_{X, Y} [\bI_{Y=1}], \label{eqn: loss_cst}
}
where $(a)$ holds because probability measure of $\{\xv: h(\xv)=1\}$ is zero in the absence of atomic components in the measure, that concludes $\Pr(h(X) =0) = 1-\Pr(h(X)=1)=1$. The proof is complete by \eqref{eqn: def_is_nondef} and \eqref{eqn: loss_cst}.
\end{proof}

\begin{lemma} \label{lem: bound_diff_discrete}
 Let $\mu_X$ be a probability measure on $\Xcal$, and let $\Hcal$ and $\Rcal$ be hypothesis classes as in Lemma \ref{lem: const_cs}. Further, let $h_1, h_2\in \Hcal$ and $r_1, r_2\in \Rcal$. Then, we have
\ea{
\big|L_{\mathrm{def}}^{\mu_X}(h_1, r_1)-L_{\mathrm{def}}^{\mu_X}(h_2, r_2)\big|\leq \big|L_{\mathrm{def}}^{\mu_d}(h_1, r_1)-L_{\mathrm{def}}^{\mu_d}(h_2, r_2)\big|,
}
where $\mu_d$ is pure atomic (discrete) probability measure on $\Xcal$.
\end{lemma}

\begin{proof}
We know that for probability measure $\mu_X$, there exists $p\in [0, 1]$ and probability measures $\mu_d$ and $\mu_{cs}$, such that
\ea{
\mu_X = p\mu_d+(1-p)\mu_{cs},
}
where $\mu_d$ is pure atomic and $\mu_{cs}$ has no atomic components. As a result, for every function $f(\cdot):\Xcal\to \Rbb$, we have
\ea{
\Ebb_{X\sim \mu_X}\big[f(X)\big]=p\Ebb_{x\sim \mu_d}\big[f(X)\big]+(1-p)\Ebb_{x\sim \mu_d}\big[f(X)\big].
}
With the same reasoning, we have
\ea{
L_{\mathrm{def}}^{\mu_X}(h, r)=pL_{\mathrm{def}}^{\mu_d}(h, r)+(1-p)L_{\mathrm{def}}^{\mu_{cs}}(h, r).
}
Next, we have
\ea{
L_{\mathrm{def}}^{\mu_X}(h_1, r_1)-L_{\mathrm{def}}^{\mu_X}(h_2, r_2) &= p\big[L_{\mathrm{def}}^{\mu_d}(h_1, r_1)-L_{\mathrm{def}}^{\mu_d}(h_2, r_2)\big]+(1-p)\big[L_{\mathrm{def}}^{\mu_{cs}}(h_1, r_1)-L_{\mathrm{def}}^{\mu_{cs}}(h_2, r_2)\big]\\
&\overset{(a)}{=} p\big[L_{\mathrm{def}}^{\mu_d}(h_1, r_1)-L_{\mathrm{def}}^{\mu_d}(h_2, r_2)\big], \label{eqn: remove_cs}
}
where $(a)$ holds because of Lemma \ref{lem: const_cs} that proves $L_{\mathrm{def}}^{\mu_{cs}}(h, r)$ is constant for every $(h, r)\in \Rcal\times \Hcal$.

Finally, using \eqref{eqn: remove_cs}, and since $p\in [0, 1]$, the proof is complete.
\end{proof}

\begin{lemma} \label{lem: diff_ub}
Let $\mathrm{supp}(h)=\max_{\Scal: \forall x\in \Scal, h(x)=1} |\Scal|$ and $\Hcal_d = \{h:\Xcal\to \{0, 1\}\,|\, \mathrm{supp}(h)\leq d\}$. Further, let $\mu_X$ be an atomic measure on $\Xcal$, and define
\ea{
\hat{h} := \argmin_{h\in \Hcal_{d(\Hcal)}} L_{0-1}^{\mu_X\mu_{Y|X}} (h), \label{eqn: htilde_def}
}
where
\ea{
L_{0-1}^{\mu_X\mu_{Y|X}} (h) = \Ebb_{\mu_{X}\mu_{Y|X}}\big[\bI_{h(X)\neq Y} \big],
}
and
\ea{
\hat{r}:= \argmin_{r\in \Hcal_{d(\Rcal)}} L_{\mathrm{def}}^{\mu_X}(\hat{h}, r). \label{eqn: rtilde_def}
}
Further, define the pair $(h^*, r^*)$ be the optimal classifier
\ea{
(h^*, r^*)=\argmin_{(h, r)\in \Hcal_{d(\Hcal)}\times \Hcal_{d(\Rcal)}} L_{\mathrm{def}}^{\mu_{X}\mu_{Y|X}}(h, r).
}
Then, if $d(\Hcal)\geq d(\Rcal)$, we have
\ea{
L_{\mathrm{def}}^{\mu_{X}\mu_{Y|X}}(\hat{h}, \hat{r}) - L_{\mathrm{def}}^{\mu_{X}\mu_{Y|X}}(h^*, r^*) \leq \min_{h\in \Hcal_{d(\Hcal)-d(\Rcal)}} L_{0-1}^{\mu'_X\mu_{Y|X}}(h)-\min_{h\in \Hcal_{d(\Hcal)}}L_{0-1}^{\mu'_X\mu_{Y|X}}(h),
}
where $\mu'_X$ is a purely atomic measure on $\Xcal$.
\end{lemma}

\begin{proof}
Firstly, using \eqref{eqn: rtilde_def}, we know that
\ea{
L_{\mathrm{def}}^{\mu_{X}\mu_{Y|X}}(\hat{h}, \hat{r}) \leq L_{\mathrm{def}}^{\mu_{X}\mu_{Y|X}}(\hat{h}, r^*).
}
Hence, we have 
\ea{
\underbrace{L_{\mathrm{def}}^{\mu_{X}\mu_{Y|X}}(\hat{h}, \hat{r}) - L_{\mathrm{def}}^{\mu_{X}\mu_{Y|X}}(h^*, r^*)}_{D} &\leq L_{\mathrm{def}}^{\mu_{X}\mu_{Y|X}}(\hat{h}, r^*) - L_{\mathrm{def}}^{\mu_{X}\mu_{Y|X}}(h^*, r^*)\\
&=\Ebb\big[\bI_{r^*(X)=0}\bI_{\hat{h}(X)\neq Y}\big] - \Ebb\big[\bI_{r^*(X)=0}\bI_{h^*(X)\neq Y}\big]. \label{eqn: ub_d_1}
}
Next, we form the conditional probability measure $\mu'_X = \mu_{X|r^*(X)=0}$. Therefore, using \eqref{eqn: ub_d_1} we have
\ea{
D = \Pr\big(r^*(X)=0\big) \big[L_{0-1}^{\mu'_X\mu_{Y|X}}(\hat{h}) - L_{0-1}^{\mu'_X\mu_{Y|X}}(h^*)\big]. \label{eqn: def_big_D}
}

Next, since $h^*\in \Hcal_{d(\Hcal)}$, we know that
\ea{
L_{0-1}^{\mu'_X\mu_{Y|X}}(h^*) \geq \min_{h\in \Hcal_{d(\Hcal)}} L_{0-1}^{\mu'_X\mu_{Y|X}}(h). \label{eqn: lb_loss_x'}
}
Moreover, we prove that
\ea{
L_{0-1}^{\mu'_X\mu_{Y|X}}(\hat{h}) \leq \min_{h\in \Hcal_{d(\Hcal)-d(\Rcal)}} L_{0-1}^{\mu'_X\mu_{Y|X}}(h). \label{eqn: ub_loss_x'}
}
We prove this inequality by contradiction. If \eqref{eqn: ub_loss_x'} is not correct, then there exists $h'\in \Hcal_{d(\Hcal)-d(\Rcal)}$, such that 
\ea{
L_{0-1}^{\mu'_X\mu_{Y|X}}(h')< L_{0-1}^{\mu'_X\mu_{Y|X}}(\hat{h}). \label{eqn: contradiction_assumption}
}
Then, we define a function $h'':\Xcal\to\{0, 1\}$ as below
\ea{
h''(\xv)=\left\{\begin{array}{c c}
     h'(\xv)&r^*(\xv)=0 \\
     \hat{h}(\xv)&r^*(\xv)=1 
\end{array}\right. .
}
Using the definition of $\Hcal_d$ and since $\mathrm{supp}(r^*)\leq d(\Rcal)$, one could show that $h''\in \Hcal_{d(\Hcal)}$. Furthermore, we have
\ea{
L_{0-1}^{\mu_{X}\mu_{Y|X}}(h'') &= \Pr\big(r^*(X)=0\big) L_{0-1}^{\mu'_X\mu_{Y|X}}(h')+\Pr\big(r^*(X)=1\big) L_{0-1}^{\mu_{X|r^*(X)=1}\mu_{Y|X}}(\hat{h})\\
&\overset{(a)}{<} \Pr\big(r^*(X)=0\big)L_{0-1}^{\mu'_X\mu_{Y|X}}(\hat{h})+\Pr\big(r^*(X)=1\big) L_{0-1}^{\mu_{X|r^*(X)=1}\mu_{Y|X}}(\hat{h})\\
&= L_{0-1}^{\mu_X\mu_{Y|X}}(\hat{h}), \label{eqn: contradiction_result}
}
where $(a)$ holds using \eqref{eqn: contradiction_assumption}, and \eqref{eqn: contradiction_result} and $\hat{h}\in \Hcal_{d(\Hcal)}$ is a contradiction of \eqref{eqn: htilde_def}.

Using \eqref{eqn: def_big_D}, \eqref{eqn: lb_loss_x'}, \eqref{eqn: ub_loss_x'}, and since $\Pr(r^*(X)=0)\leq 1$, the proof is complete.
\end{proof}

\begin{lemma} \label{lem: optimal_Hd}
Let $\mu_{X}$ be a purely atomic measure on $\Xcal$. Further, let $\{\xv_{i, 1}\}_i$ be the points in $\Xcal$ for which we have
\ea{
\Pr(Y=1|X=\xv_{i, 1})\leq \Pr(Y=0|X=\xv_{i, 1}), \label{eqn: assumption_1}
}
and without loss of generality, assume that $\{\xv_{i, 2}\}$ are the points for which we have
\ea{
\Pr(Y=1|X=\xv_{i, 2})> \Pr(Y=0|X=\xv_{i, 2}),
}
and if $i<j$, then we have
\ea{
&\Pr(X=\xv_{i, 2})\big[\Pr(Y=1|X=\xv_{i, 2})-\Pr(Y=0|X=\xv_{i, 2})\big]\nonumber\\&\geq \Pr(X=\xv_{j, 2})\big[\Pr(Y=1|X=\xv_{j, 2})-\Pr(Y=0|X=\xv_{j, 2})\big]. \label{eqn: decreasing_order}
}

Then, we have
\ea{
\min_{h\in \Hcal_d} L_{0-1}^{\mu_X\mu_{Y|X}}(h) &= \sum_{i}\Pr(\xv_{i, 1}) \Pr(Y=1|X=\xv_{i, 1})+\sum_{i=1}^{d}\Pr(\xv_{i, 2})\Pr(Y=0|X=\xv_{i, 2}) \nonumber\\&~~~+\sum_{i=d+1}^{\infty} \Pr(\xv_{i, 2})\Pr(Y=1|X=\xv_{i, 2}), \label{eqn: min_Hd}
}
where $\Hcal_d$ is defined as in Lemma \ref{lem: diff_ub}.
\end{lemma}
\begin{proof}
Let $h^*$ be the optimal classifier
\ea{
h^* = \argmin_{h\in \Hcal_d} L_{0-1}^{\mu_X\mu_{Y|X}}(h).\label{eqn: optimal_class}
}

Then, firstly, either $h(\xv_{i, 1})=0$, or $\Pr(Y=1|X=\xv_{i, 1})=\Pr(Y=0|X=\xv_{i, 1})$ for all $i$s. We prove this claim by contradiction. If for some $i$, we have $h(\xv_{i, 1})=1$, and $\Pr(Y=1|X=\xv_{i, 1})\neq\Pr(Y=0|X=\xv_{i, 1})$, then we could define $h'$ as
\ea{
h'(\xv)=\left\{\begin{array}{cc}
     h^*(\xv)& \xv\neq \xv_{i, 1}  \\
     0& \xv = \xv_{i, 1}
\end{array}\right. .
}
One could see that $h'\in \Hcal_{d}$, and that 
\ea{
L_{0-1}^{\mu_{X}\mu_{Y|X}}(h')-L_{0-1}^{\mu_{X}\mu_{Y|X}}(h^*) = \Pr(\xv_{i, 1}) \big[\Pr(Y=1|X=\xv_{i, 1})-\Pr(Y=0|X=\xv_{i, 1})\big]\overset{(a)}{<}0, \label{eqn: contradiction_optimality}
}
where $(a)$ holds by \eqref{eqn: assumption_1} and since $\big[\Pr(Y=1|X=\xv_{i, 1})\neq\Pr(Y=0|X=\xv_{i, 1})\big]$.  The inequality \eqref{eqn: contradiction_optimality} has contradiction with \eqref{eqn: optimal_class}.

As a result, by forming a set $\Scal$ of indices $i$ for which $h^*(\xv_{i, 1})=1$, we have
\ea{
\min_{h\in \Hcal_d} L_{0-1}^{\mu_{X}\mu_{Y|X}}(h)&= \sum_{i\notin \Scal} \Pr(\xv_{i, 1})\Pr(Y=1|X=\xv_{i, 1})+ \sum_{i\in \Scal}\Pr(\xv_{i, 1})\Pr(Y=0|X=\xv_{i, 1})\nonumber\\
&~~~~+\sum_{i}\Pr(\xv_{i, 2})\Pr(Y=1|X=\xv_{i, 2})\nonumber\\
&~~~~+\min_{|S|}\min_{h\in \Hcal_{d-|\Scal|}} \sum_{i} \bI_{h(\xv_{i, 2})=1}\Pr(\xv_{i, 2})\big[\Pr(Y=0|X=\xv_{i, 2})-\Pr(Y=1|X=\xv_{i, 2})\big]\\
&\overset{(a)}{=} \sum_{i} \Pr(\xv_{i, 1})\Pr(Y=1|X=\xv_{i, 1})+\sum_{i}\Pr(\xv_{i, 2})\Pr(Y=1|X=\xv_{i, 2})\nonumber\\
&~~~~+\min_{h\in \Hcal_d}\sum_{i} \bI_{h(\xv_{i, 2})=1}\Pr(\xv_{i, 2})\big[\Pr(Y=0|X=\xv_{i, 2})-\Pr(Y=1|X=\xv_{i, 2})\big]\\
&\overset{(b)}{=}\sum_{i} \Pr(\xv_{i, 1})\Pr(Y=1|X=\xv_{i, 1})+\sum_{i}\Pr(\xv_{i, 2})\Pr(Y=1|X=\xv_{i, 2})\nonumber\\
&~~~~-\max_{\Pcal: |\Pcal|\leq d} \sum_{i\in \Pcal}\Pr(\xv_{i, 2})\big[\Pr(Y=0|X=\xv_{i, 2})-\Pr(Y=1|X=\xv_{i, 2})\big]\\
&\overset{(c)}{=}\sum_{i} \Pr(\xv_{i, 1})\Pr(Y=1|X=\xv_{i, 1})+\sum_{i}\Pr(\xv_{i, 2})\Pr(Y=1|X=\xv_{i, 2})\nonumber\\
&~~~~+\sum_{i=1}^d\Pr(\xv_{i, 2})\big[\Pr(Y=0|X=\xv_{i, 2})-\Pr(Y=1|X=\xv_{i, 2})\big], \label{eqn: final_Hd_optimal}
}
where $(a)$ holds using that for $i\in \Scal$ we have $\big[\Pr(Y=1|X=\xv_{i, 1})=\Pr(Y=0|X=\xv_{i, 1})\big]$, and since $\Hcal_{d-|\Scal|}\subseteq \Hcal_{d}$. Further, $(b)$ holds by the definition of $\Hcal_{d}$ in which $\mathrm{supp}(h)$ is assumed to be bounded by $d$, and, $(c)$ holds using the assumption \eqref{eqn: decreasing_order}. Finally, one could see that \eqref{eqn: final_Hd_optimal} is equal to \eqref{eqn: min_Hd}. 
\end{proof}
\begin{lemma} \label{lem: ordered}
For an ordered probability mass function
\ea{
p_1\leq p_2\leq  \ldots\leq p_n,
}
on a finite set, and for $a\in\{1, \ldots, n\}$, we have
\ea{
\sum_{i=1}^{a}p_i\leq \frac{a}{n}.
}
\begin{proof}
We prove this lemma by contradiction. Assume that
\ea{
\sum_{i=1}^{a} p_i> a.
}
Since $p_i$s are ordered, one could see that for all sets $\Scal_t\subseteq \{1, \ldots, n\}$ with $|\Scal_t|=a$, we have
\ea{
\sum_{i\in \Scal_t} p_i >\frac{a}{n}.
}
We know that $n\choose a$ number of such distinct sets exist. Hence, we have
\ea{
\sum_{t=1}^{a \choose n} \sum_{i\in \Scal_t} p_i>\frac{a}{n} {n \choose a}.\label{eqn: combination_ub}
}
Moreover, one could see that for each $i$, $p_i$ is repeated in LHS of \eqref{eqn: combination_ub} for ${n-1 \choose {a-1}}$ times. Consequently, we see that
\ea{
{{n-1}\choose {a-1}} = \sum_{j=1}^{{n-1}\choose {a-1}} \sum_{i=1}^n p_i >\frac{a}{n} {n \choose a} = {{n-1}\choose {a-1}},
}
that is a contradiction.
\end{proof}
\end{lemma}
\begin{proof}[Proof of Theorem \ref{thm: vc}]
We derive the lower- and upper-bound in two steps as follows.

\noindent
	$\bullet$ {\bfseries Lower-bound}:
	To prove the lower-bound, for every hypothesis class $\Hcal$ and $\Rcal$, we design a distribution of $(x, y, m)$ such that 
\ea{
\defl(h^*, r^*)=0,
}
while
\ea{
\defl(\tilde{h}, \tilde{r}) = \frac{1}{d(\Hcal)+1}.
}

For every $d(\Hcal)+1$ samples $\Xv=(\xv_1, \ldots, \xv_{d(\Hcal)+1})$, using the definition of VC dimension, we can find labels $\yv = (y_1, \ldots, y_{d(\Hcal)+1})$ such that no classifier $h\in \Hcal$ can obtain them (i.e., there is no $h$ such that $h(\xv_i)=y_i$, for $i\in [1:d(\Hcal)+1]$). We set
\ea{
p(\xv_i) =\left\{\begin{array}{cc}
    \frac{1+\ep}{d(\Hcal)+1} & i=1 \\
     \frac{1}{d(\Hcal)+1}& i\in [2:d(\Hcal)]\\
     \frac{1-\ep}{d(\Hcal)+1} & i=d(\Hcal)+1
\end{array}\right.,
}
\ea{
p(y|\xv_i)= \left\{\begin{array}{cc}
     1& y= y_i, \\
     0 & o.w.
\end{array}\right.,
}
and
\ea{
p(m|\xv_i, y)= \left\{\begin{array}{cc}
     1& m=y_i, i=1, \\
     1 & m=1-y_i, i=d(\Hcal)+1,\\
     0& o.w.
\end{array}\right. .}

If we train $\hat{h}$ and $\hat{r}$ separately, it means
\ea{
\hat{h} &= \argmin_{h\in \Hcal} \Ebb_{(\xv, y, m)\sim p}  [{\bI}_{h(\xv) \neq y} ] \\
&= \argmin_{h\in \Hcal} \frac{1+\ep}{d(\Hcal)+1} \bI_{h(\xv_1)\neq y_1} +\sum_{i=2}^{d(\Hcal)} \frac{1}{d(\Hcal)+1} \bI_{h(\xv_i)\neq y_i} + \frac{1-\ep}{d(\Hcal)+1} \bI_{h(\xv_{d(\Hcal)+1})\neq y_{d(\Hcal)+1}}.\label{eqn: sum_l_h}
}

By the definition of $\yv$, we know that at least one of the terms in the RHS of \eqref{eqn: sum_l_h} is non-zero. In such case, for every subset $T$ of $[1:d(\Hcal)+1]$ of size $d(\Hcal)$, one can find $h\in \Hcal$, such that $h(\xv_i)=y_i$ for $i\in T$. Hence, to minimize RHS of \eqref{eqn: sum_l_h}, we should have $\hat{h}(\xv_i)\neq y_i$ only for $i=d(\Hcal)+1$.

Further, $\hat{r}$ is obtained as 
\ea{
\hat{r} = \argmin_{r\in \Rcal} \Ebb_{(\xv, y, m)\sim p} [{\bI}_{\hat{h}(\xv) \neq y }{\bI}_{r(\xv) = 0 } + {\bI}_{m \neq y} {\bI}_{r(\xv)=1}]. \label{eqn: min_l_r}
}
By the definition of $p(m|y, \xv)$ and $\hat{h}$, we can rewrite \eqref{eqn: min_l_r} as 
\ea{
\hat{r} = \argmin_{r\in \Rcal} \frac{1-\ep}{d(\Hcal)+1} \bI_{r(\xv_{d+1})=1}+\frac{1-\ep}{d(\Hcal)+1} \bI_{r(\xv_{d+1})=0}.
}
One can see that by any choice of $\hat{r}(\cdot)$, we have
\ea{
\defl(\hat{h}, \hat{r}) = \frac{1-\ep}{d(\Hcal)+1}.
}
Finally, by the arbitrariness of $\ep$, we have
\ea{
\defl(\hat{h}, \hat{r}) = \frac{1}{d(\Hcal)+1}.
}

Further, we prove that $\defl(h^*, r^*)=0$ by constructing $h^*$ and $r^*$. Since $d(\Rcal)\geq 2$, we can shatter $\{\xv_1, \xv_{d(\Hcal)+1}\}$ by $\Rcal$, which means that there exists $r^*\in \Rcal$ such that $r(\xv_1)=1$, and $r(\xv_{d(\Hcal)+1})=0$. As a result, we have
\ea{
\defl (h^*, r^*) = \sum_{i=2}^{d(\Hcal)} \frac{1}{d(\Hcal)+1}\bI_{r^*(\xv_i)=0}\bI_{h^*(\xv_i)\neq y_i}+\frac{1-\ep}{d(\Hcal)+1}\bI_{h^*(\xv_{d(\Hcal)+1})\neq y_i}.\label{eqn: l01_joint}
}
Since VC dimension of $\Hcal$ is $d(\Hcal)$, we can find $h^*$ such that all terms in the RHS of \eqref{eqn: l01_joint} is zero. Hence, we have $\defl(h^*, r^*)=0$, that completes the proof.

\noindent
	$\bullet$ {\bfseries Upper-bound}: For $d(\Hcal)\leq d(\Rcal)$ the upper-bound is trivial. Then, we asssume $d(\Hcal)>d(\Rcal)$. Let $D_{\mu_{XYM}}$ be
	\ea{
	D_{\mu_{XYM}}^{\Hcal, \Rcal} = L_{\mathrm{def}}^{\mu_X\mu_{Y|X}}(\hat{h}, \hat{r}) - L_{\mathrm{def}}^{\mu_X\mu_{Y|X}}(h^*, r^*).
	}
	To upper-bound $\inf_{\Hcal, \Rcal}\sup_{\mu_{XYM}} D_{\mu_{XYM}}^{\Hcal, \Rcal}$, we find a pair of hypothesis classes $\Hcal$ and $\Rcal$, such that for all joint probability measures $\mu_{XYM}$, we have $D_{\mu_{XYM}}^{\Hcal, \Rcal}\leq \frac{d(\Rcal)}{d(\Hcal)}$.
	
	We choose $\Hcal = \Hcal_{d(\Hcal)}$, and $\Rcal = \Hcal_{d(\Rcal)}$, where $\Hcal_{d}$ is defined in Lemma \ref{lem: diff_ub}. One could check that $VC(\Hcal)= d(\Hcal)$, and $VC(\Rcal)=d(\Rcal)$.  Further, using Lemma \ref{lem: bound_diff_discrete}, we know that $D_{\mu_{XYM}}^{\Hcal, \Rcal}$ is bounded by $D_{\mu'_{XYM}}^{\Hcal, \Rcal}$, in which $\mu'_X$ is purely atomic. For such measures, Lemma \ref{lem: diff_ub} proves that
	\ea{
	D_{\mu_{XYM}}^{\Hcal, \Rcal} \leq \min_{h\in \Hcal_{d(\Hcal)-d(\Rcal)}} L_{0-1}^{\mu_X'\mu_{Y|X}}(h)-\min_{h\in \Hcal_{d(\Hcal)}} L_{0-1}^{\mu'_X\mu_{Y|X}}(h).
	}
	As a result, we have
	\ea{
	\sup_{\mu_{XYM}}D_{\mu_{XYM}}^{\Hcal, \Rcal}&\leq \sup_{\mu_{XYM}: \, \mu_X\, \text{ atomic}}D_{\mu_{XYM}}^{\Hcal, \Rcal} \\&\leq \sup_{\mu_{XY}: \, \mu_X\, \text{ atomic}}\min_{h\in \Hcal_{d(\Hcal)-d(\Rcal)}} L_{0-1}^{\mu_X\mu_{Y|X}}(h)-\min_{h\in \Hcal_{d(\Hcal)}} L_{0-1}^{\mu_{X}\mu_{Y|X}}(h).
	}
	Next, by applying Lemma \ref{lem: optimal_Hd}, we have
	\ea{
	\sup_{\mu_{XYM}}D_{\mu_{XYM}}^{\Hcal, \Rcal}\leq \sup_{\mu_{XY}: \, \mu_X\, \text{ atomic}} \sum_{i=d(\Hcal)-d(\Rcal)+2}^{d(\Hcal)} \Pr(x_{i, 2}) \big[\Pr(Y=1|X=x_{i, 2})-\Pr(Y=0|X=x_{i, 2})\big], \label{eqn: last_ub}
	}
	where $\{x_{i, 2}\}_i$ are defined in Lemma \ref{lem: optimal_Hd}. 
	
	Since $\Pr(Y=1|X=x_{i, 2})> \Pr(Y=0|X=x_{i, 2})$ we could define 
	\ea{
	q_i = \frac{\Pr(x_{i, 2}) \big[\Pr(Y=1|X=x_{i, 2})-\Pr(Y=0|X=x_{i, 2})\big]}{\sum_{j=1}^{d(\Hcal)}\Pr(x_{j, 2}) \big[\Pr(Y=1|X=x_{j, 2})-\Pr(Y=0|X=x_{j, 2})\big]}.
	}
	Then, by the definition of $x_{i, 2}$ we know that
	\ea{
	q_1\geq q_2\geq \ldots \geq q_{d(\Hcal)},
	}
	and $\sum_{i=1}^{d(\Hcal)} q_i = 1$. Hence, using Lemma \ref{lem: ordered} we have
	\ea{
	\frac{\sum_{j=d(\Hcal)-d(\Rcal)+1}^{d(\Hcal)}\Pr(x_{j, 2}) \big[\Pr(Y=1|X=x_{j, 2})-\Pr(Y=0|X=x_{j, 2})\big]}{\sum_{j=1}^{d(\Hcal)}\Pr(x_{j, 2}) \big[\Pr(Y=1|X=x_{j, 2})-\Pr(Y=0|X=x_{j, 2})\big]} = \sum_{j=d(\Hcal)-d(\Rcal)+1}^{d(\Hcal)} q_i \leq \frac{d(\Rcal)}{d(\Hcal)},
	}
	which concludes that
	\ea{
	&\sum_{j=d(\Hcal)-d(\Rcal)+1}^{d(\Hcal)}\Pr(x_{j, 2}) \big[\Pr(Y=1|X=x_{j, 2})-\Pr(Y=0|X=x_{j, 2})\big]\nonumber\\&\leq \frac{d(\Rcal)}{d(\Hcal)} \sum_{j=1}^{d(\Hcal)}\Pr(x_{j, 2}) \big[\Pr(Y=1|X=x_{j, 2})-\Pr(Y=0|X=x_{j, 2})\big]\leq \frac{d(\Rcal)}{d(\Hcal)}.
	}
	This, together with \eqref{eqn: last_ub} completes the proof.
\end{proof}

\section{Proof of Proposition \ref{prop: gen_labeled_unlabeled}}\label{app: gen_labeled_unlabeled}

We will prove the following proposition from which Proposition \ref{prop: gen_labeled_unlabeled} can be obtained from by re-arranging the terms.

Let $\Scal_{l}=\{(\xv_i, y_i, m_i)\}_{i=1}^{n_l}$ and $\Scal_{u}=\{(\xv_{i+n_l}, y_{i+n_l})\}_{i=1}^{n_u}$ be two iid sample sets that are drawn from the joint distribution $P_{X, Y, M}$ and are labeled and not labeled by human, respectively. 
Assume that the optimal classifier $\bar{h}=\argmin_{h}\Ebb_{X, Y\sim \mu_{XY}}[\bI_{h(X)\neq Y}]$ is a member of $\Hcal$ (i.e., realizability).
Then, with probability at least $1-\delta$ we have
\begin{align}
&L_{\mathrm{def}}^{0{-}1}(\hat{r}, \hat{h})\leq
	L_{\mathrm{0-1}}(h^*,r^*)+ \bm{\Rad_{n_u}(\Hcal)}+2\Rad_{n_l}(\Rcal)\nonumber\\&+2\min\big\{\Pr(M\neq Y), \Rad_{n_l\Pr(M\neq Y)/2}(\Rcal)\big\}+C\sqrt{\frac{\log 1/\delta}{n_l}}+e^{-n_{l}\Pr(M\neq Y)^2/2} \nonumber\\
	&+{\bm{C'\sqrt{\frac{\log 1/\delta}{n_u}}}}\label{eqn: gen_labeled_unlabeled_app}
\end{align}
where  $h^*, r^*=\arg\min_{(h, r)\in \Hcal\times \Rcal} L_{0-1}(h, r)$.

Compare this to using only $S_l$ to learn jointly $\tilde{h}, \tilde{r}$ we get \cite{mozannar2020consistent}\footnote{Note that in \cite{mozannar2020consistent}, they set the notation in a manner that $r\in \{-1, 1\}$. Hence, $\Rad_{n}(\Rcal)$ under such notation is twice as much as the case in this paper (i.e., $r\in \{0, 1\}$). Here, we express their results with our choice of notation.}:
\begin{align}
&L_{\mathrm{def}}^{0{-}1}(\tilde{r}, \tilde{h})\leq
	L_{\mathrm{0-1}}(h^*,r^*)+\bm{\Rad_{n_l}(\Hcal)}+2\Rad_{n_l}(\Rcal)\nonumber\\&+2\Rad_{n_l\Pr(M\neq Y)/2}(\Rcal)+C'\sqrt{\frac{\log 1/\delta}{n_l}}\nonumber\\&+\tfrac{\Pr(M\neq Y)}{2}e^{-n\Pr(M\neq Y)/2}
\end{align}

We start by introducing some useful lemmas, and then we continue with the proof of proposition.

\begin{lemma} \label{lem: defer_less}
Let $h^*(x)=\argmin_{h\in \Fcal} L_{0-1}(h)$, where $\Fcal$ is the class of all functions $h(\cdot): \Xcal\to \Ycal$. Then, for every function $r(\cdot):\Xcal\to \{0, 1\}$, we have
\ea{
\Ebb_{X, Y} \big[\bI_{r(X)=0} \bI_{h^*(x)\neq y}\big]\leq \Ebb_{X, Y} \big[\bI_{r(X)=0} \bI_{h(x)\neq y}\big],
}
for all function $h(\cdot)\in \Fcal$.
\end{lemma}

\begin{proof}
Since $h^*(\cdot)$ could be any function, it is easy to show that for $x\in \Dcal$, where $\Dcal = \{x: f_X(x)\neq 0\}$, we have
\ea{
h^*(x)=\argmin_{v} \Ebb_{Y|X=x}[\bI_{v\neq Y}],
}
which concludes that 
\ea{
\Ebb_{Y|X=x}[\bI_{h^*(x)\neq Y}]\leq \Ebb_{Y|X=x}[\bI_{h(x)\neq Y}],
}
for all $h(\cdot)\in \Fcal$. Hence, we have
\ea{
\Ebb_{X, Y}[\bI_{r(X)=0}\bI_{h^*(X)\neq Y}]&=\Ebb_{X}\big[\bI_{r(X)=0}\Ebb_{Y|X=x}[\bI_{h^*(x)\neq Y}]\big]\\&\leq \Ebb_{X}\big[\bI_{r(X)=0}\Ebb_{Y|X=x}[\bI_{h(x)\neq Y}]\\&=\Ebb_{X, Y}[\bI_{r(X)=0}\bI_{h(X)\neq Y}],
}
which completes the proof.

\end{proof}

\begin{lemma} \label{lem: r_h_lem}
Let $h^*(x)=\argmin_{h\in \Fcal} L_{0-1}(h)$, where $\Fcal$ is the class of all functions $h(\cdot): \Xcal\to \Ycal$. If we have $h^*(\cdot)\in \Hcal$, then there exists $r\in \Rcal$ such that the pair $(h^*, r)$ is a minimizer of the optimization problem
\ea{
\argmin_{(h, r)\in \Hcal\times \Rcal} \Ebb_{X, Y, M}\big[\bI_{r(X)=0}\bI_{h(X)\neq Y}+\bI_{r(X)=1} \bI_{M\neq Y}\big].
}
\end{lemma}
\begin{proof}
We prove this lemma by showing 
\ea{
\min_{(h, r)\in \Hcal\times \Rcal} \Ebb_{X, Y, M}\big[\bI_{r(X)=0}&\bI_{h(X)\neq Y}+\bI_{r(X)=1} \bI_{M\neq Y}\big] \nonumber\\&= \min_{r\in \Rcal} \Ebb_{X, Y}\big[\bI_{r(X)=0}\bI_{h^*(X)\neq Y}+\bI_{r(X)=1} \bI_{M\neq Y}\big]. \label{eqn: equality_two_mins}
}
To show \eqref{eqn: equality_two_mins}, using Lemma \ref{lem: defer_less}, we know that 
\ea{
&\min_{(h, r)\in \Hcal\times \Rcal} \Ebb_{X, Y, M}\Big[\bI_{r(X)=0}\bI_{h(X)\neq Y}+\bI_{r(X)=1} \bI_{M\neq Y}\Big]\nonumber\\&
\geq \min_{(h, r)\in \Hcal\times \Rcal}\Ebb_{X, Y, M}[\bI_{r(X)=0}\bI_{h^*(X)\neq Y}]+\Ebb_{X, Y, M}\Big[\bI_{r(X)=1}\bI_{M\neq Y}\Big]\\
&= \min_{r\in \Rcal} \Ebb_{X, Y, M}[\bI_{r(X)=0}\bI_{h^*(X)\neq Y}]+\Ebb_{X, Y, M}\Big[\bI_{r(X)=1}\bI_{M\neq Y}\Big]. \label{eqn: lbound_min_h_r}
}

On the other hand, using the minimum property, one could show that
\ea{
&\min_{h\in \Hcal}\min_{r\in \Rcal} \Ebb\Big[\bI_{r(X)=0}\bI_{h(X)\neq Y} +\bI_{r(X)=1}\bI_{M\neq Y}\Big]\nonumber\\
&\leq \min_{r\in \Rcal} \Ebb_{X, Y, M}[\bI_{r(X)=0}\bI_{h^*(X)\neq Y}]+\Ebb_{X, Y, M}\Big[\bI_{r(X)=1}\bI_{M\neq Y}\Big].\label{eqn: ubound_min_h_r}
}
Hence, using the lower- and upper-bound in \eqref{eqn: lbound_min_h_r} and \eqref{eqn: ubound_min_h_r}, one could show \eqref{eqn: equality_two_mins} and complete the proof.
\end{proof}

\begin{proof}[Proof of Proposition \ref{prop: gen_labeled_unlabeled}]
We prove  \eqref{eqn: gen_labeled_unlabeled_app} in three steps: (i) we bound the expected $0-1$ loss of the classifier $\hat{h}$ when deferral does not happen by a function of the optimal expected $0-1$ loss in such cases, (ii) we bound the joint loss $L_{\mathrm{def}}$ by a function of the optimal joint loss and the Rademacher complexity of a hypothesis class, and (iii) we bound the Rademacher complexity of the aforementioned class by the Rademacher complexity of the deferral hypothesis class $\Rcal$.

\noindent
	$\bullet$ {\bfseries Step (i)}: Using Rademacher inequality (Theorem 3.3 of \cite{mohri2018foundations}), with probability $1-\delta/4$, we have
	\ea{
	L_{0-1}(\hat{h})\leq \hat{L}_{0-1}(\hat{h})+2\Rad_{n_u}(\Gcal)+\sqrt{\frac{\log 4/\delta}{2n_u}},\label{eqn: gen_classifier}
	}
	where $\Gcal = \{\xv, y\to \bI_{h(\xv)\neq y}\,:\, h\in \Hcal\}$.
	
	Furthermore, using \eqref{eqn: gen_classifier}, since $\hat{h}$ is an optimizer of the empirical loss in $\Hcal$ and since $h^*\in \Hcal$, with probability $1-\delta/2$ we have
	\ea{
	L_{0-1}(\hat{h}) &\leq \hat{L}_{0-1}(h^*)+2\Rad_{n_u} (\Gcal)+\sqrt{\frac{\log 4/\delta}{2n_u}}\\
	&\overset{(a)}{\leq} L_{0-1}(h^*)+2\Rad_{n_u}(\Gcal)+\frac{3\sqrt{2}}{2}\sqrt{\frac{\log 4/\delta}{n_u}},\label{eqn: gen_class_mc}
	}
	where $(a)$ holds using McDiarmid's inequality, union bound, and by that the empirical loss is $\frac{2}{n}$-bounded difference.
	
	Next, using Lemma 3.4 of \cite{mohri2018foundations} we know that $\Rad_n(\Gcal) = \frac{1}{2}\Rad_n(\Hcal)$. By means of such identity and \eqref{eqn: gen_class_mc}, with probability $1-\delta/2$ we have
	\ea{
	L_{0-1}(\hat{h})\leq L_{0-1}(h^*)+\Rad_n(\Hcal)+\frac{3\sqrt{2}}{2}\sqrt{\frac{\log 4/\delta}{n_u}}. \label{eqn: class_rad_h}
	}
	
	It remains to show that for each function $r(\cdot):\Xcal\to \{0, 1\}$, we could bound $\Ebb_{X, Y}[\bI_{r(X)=0}\bI_{\hat{h}(X)\neq Y}]$ by sum of $\Ebb_{X, Y}[\bI_{r(X)=0}\bI_{h^*(X)\neq Y}]$ and a term that is corresponded to the concentration of measure for large sample size. For proving such inequality, first we know that
	\ea{
	L_{0-1}(h^*) &= \Ebb_{X, Y}\big[\bI_{r(X)=0}\bI_{h^*(X)\neq Y}\big]+\Ebb_{X}\Big[\bI_{r(X)=1} \Ebb_{Y|X=\xv}[\bI_{h^*(X)\neq Y}]\Big]\\
	&\overset{(a)}{\leq }\Ebb\big[\bI_{r(X)=0}\bI_{h^*(X)\neq Y}\big]+\Ebb_{X}\Big[\bI_{r(X)=1} \Ebb_{Y|X=\xv}[\bI_{h(X)\neq Y}]\Big], \label{eqn: bound_0-1}
	}
	for all $h\in \Fcal$, where $(a)$ is followed by Lemma \ref{lem: defer_less}.
	
	Using \eqref{eqn: class_rad_h} and \eqref{eqn: bound_0-1}, we have
	\ea{
    &\Ebb_{X, Y}\big[\bI_{r(X)=0}\bI_{\hat{h}(X)\neq Y}\big]+\Ebb_{X}\Big[\bI_{r(X)=1} \Ebb_{Y|X=\xv}[\bI_{\hat{h}(X)\neq Y}]\Big]\nonumber\\&{\leq }\Ebb\big[\bI_{r(X)=0}\bI_{h^*(X)\neq Y}\big]+\Ebb_{X}\Big[\bI_{r(X)=1} \Ebb_{Y|X=\xv}[\bI_{\hat{h}(X)\neq Y}]\Big] + \Rad_{n_u}(\Hcal)+\frac{3\sqrt{2}}{2}\sqrt{\frac{\log 4/\delta}{n_u}},
	}
	which concludes
	\ea{
	\Ebb_{X, Y}\big[\bI_{r(X)=0}\bI_{\hat{h}(X)\neq Y}\big]{\leq }\Ebb\big[\bI_{r(X)=0}\bI_{h^*(X)\neq Y}\big] + \Rad_{n_u}(\Hcal)+\frac{3\sqrt{2}}{2}\sqrt{\frac{\log 4/\delta}{n_u}}.
	}
	
\noindent
	$\bullet$ {\bfseries Step (ii)}: We know that $\hat{r}(\cdot)$ is obtained as 
	\ea{
	\hat{r}(x)= \argmin_{r\in \Rcal} \frac{1}{n_l} \sum_{i=1}^{n_l}\big[\bI_{r(\xv_i)=0}\bI_{h(\xv_i)\neq y_i}+\bI_{r(\xv_i)=1}\bI_{m_i\neq y_i}\big],
	}
	or equivalently,
	\ea{
	\hat{r}(x) = \argmin_{r\in \Rcal} \frac{1}{n_l} \sum_{i=1}^{n_l}\big[\bI_{r(\xv_i)=0}[\bI_{h(\xv_i)\neq y_i}-\bI_{m_i\neq y_i}]\big].
	}
	
	Hence, using Rademacher inequality (Theorem 3.3 of \cite{mohri2018foundations}), with probability $1-\delta/4$, we have
	\ea{
	\Ebb_{X, Y, M}\big[\bI_{\hat{r}(X)=0}[\bI_{\hat{h}(X)\neq Y}-\bI_{M\neq Y}]\big]\leq \frac{1}{n_l} \sum_{i=1}^{n_l} \bI_{\hat{r}(\xv_i)=0}[\bI_{\hat{h}(\xv_i)\neq y_i}-\bI_{m_i\neq y_i}] +2\Rad_{n_l}(\Jcal) +\sqrt{\frac{\log 4/\delta}{n_l}},\label{eqn: gen_diff}
	}
	where \ea{\Jcal =\big\{\xv, y, m\to \bI_{r(\xv)=0}[\bI_{\hat{h}(\xv)\neq y}-\bI_{m\neq y}]\, :\, r\in \Rcal\big\}.\label{eqn: jcal_def}}
	
	Using Lemma \ref{lem: r_h_lem}, we know that there exists $r^*\in \Rcal$ such that $(r^*, h^*)$ are the minimizers of the joint loss $\defl(h, r)$ in $\Hcal\times \Rcal$. Next, since $\hat{r}$ is the minimizer of the empirical joint loss given the classifier be $\hat{h}$, and using \eqref{eqn: gen_diff}, we have
	\ea{
	\Ebb_{X, Y, M}\big[\bI_{\hat{r}(X)=0}[\bI_{\hat{h}(X)\neq Y}-\bI_{M\neq Y}]\big]\leq \frac{1}{n_l} \sum_{i=1}^{n_l} \bI_{r^*(\xv_i)=0}[\bI_{\hat{h}(\xv_i)\neq y_i}-\bI_{m_i\neq y_i}] +2\Rad_{n_l}(\Jcal) +\sqrt{\frac{\log 4/\delta}{n_l}}, \label{eqn: gen_r_*}
	}
	for $r^*(\cdot)$ defined as above.
	
	Next, using McDiarmid's inequality, union bound, and since the empirical loss in RHS of \eqref{eqn: gen_r_*} is $\frac{2}{n}$-bounded difference, then with probability at least $1-\delta/2$ we have
	\ea{
	\Ebb_{X, Y, M}\big[\bI_{\hat{r}(X)=0}[\bI_{\hat{h}(X)\neq Y}-\bI_{M\neq Y}]\big]\leq &\Ebb_{X, Y, M}\big[\bI_{r^*(X)}[\bI_{\hat{h}(X)\neq Y}-\bI_{M\neq Y}]\big]+2\Rad_n (\Jcal) \nonumber\\&+(\sqrt{2}+1)\sqrt{\frac{\log 4/\delta}{n_l}}.
	}
	
	Therefore, using step (i), and by means of union bound, one could prove that with probability at least $1-\delta$ we have
	\ea{
	\Ebb_{X, Y, M}\big[\bI_{\hat{r}(X)}[\bI_{\hat{h}(X)\neq Y} - \bI_{M\neq Y}]\big]\leq &\Ebb_{X, Y, M}\big[\bI_{r^*(X)}[\bI_{h^*(X)\neq Y} - \bI_{M\neq Y}]\big]+\Rad_{n_u}(\Hcal)+2\Rad_{n_l}(\Jcal)\nonumber\\&+\frac{3\sqrt{2}}{2}\sqrt{\frac{\log 4/\delta}{n_u}}+(\sqrt{2}+1)\sqrt{\frac{\log 4/\delta}{n_l}}, 
	}
	or equivalently
	\ea{
	\defl(\hat{r}, \hat{h})\leq
	&\defl(h^*,r^*)+\Rad_{n_u}(\Hcal)+2\Rad_{n_l}(\Jcal)+\frac{3\sqrt{2}}{2}\sqrt{\frac{\log 4/\delta}{n_u}}+(\sqrt{2}+1)\sqrt{\frac{\log 4/\delta}{n_l}}, \label{eqn: gen_without_j}
	}

\noindent
	$\bullet$ {\bfseries Step (iii)}: In this step, we bound $\Rad_n(\Gcal)$ to complete the proof. By recalling the definition of $\Jcal$ in \eqref{eqn: jcal_def}, we bound $\Rad_n(\Jcal)$ as
	\ea{
	\Rad_n(\Jcal)&=\Ebb_{\{(\xv_i, y_i, m_i)\}_{i=1}^n}\Ebb_{\siv}\big[\frac{1}{n}\sup_{g\in \Jcal}\sum_{i=1}^n\sigma_i g(\xv_i, y_i, m_i)\big]\\
	&=\Ebb_{\{(\xv_i, y_i, m_i)\}_{i=1}^n}\Ebb_{\siv}\big[\frac{1}{n}\sup_{r\in \Rcal} \sum_{i=1}^n \sigma_i \big[\bI_{r(\xv_i)=0}(\bI_{\hat{h}(\xv_i)\neq y_i}-\bI_{m_i\neq y_i})\big]\big]\\
	&\overset{(a)}{\leq} \Ebb_{\{(\xv_i, y_i, m_i)\}_{i=1}^n}\Ebb_{\siv}\big[\frac{1}{n}\sup_{r\in \Rcal} \sum_{i=1}^n \sigma_i \big[r(\xv_i)\bI_{\hat{h}(\xv_i)\neq y_i}\big]\big]\nonumber\\&\phantom{\overset{(a)}{\leq}}+\Ebb_{\{(\xv_i, y_i, m_i)\}_{i=1}^n}\Ebb_{\siv}\big[\frac{1}{n}\sup_{r\in \Rcal} \sum_{i=1}^n \sigma_i \big[r(X)\bI_{m_i\neq y_i}\big]\big]\\
	&\overset{(b)}{\leq} \Rad_n(\Rcal)+\Ebb_{\{(\xv_i, y_i, m_i)\}_{i=1}^n}\Ebb_{\siv}\big[\frac{1}{n}\sum_{i=1}^{n}\sigma_i\bI_{\hat{h}(\xv_i)\neq y_i}\big]+\underbrace{\Ebb_{\{(\xv_i, y_i, m_i)\}_{i=1}^n}\Ebb_{\siv}\big[\frac{\sum_{i=1}^n\bI_{m_i\neq y_i}}{n}\hat{\Rad}_{\Scal}(\Rcal)\big]}_{A}\\
	&\overset{(c)}{=} \Rad_n (\Rcal)+A, \label{eqn: rad_j}
	}
	where $\Scal=\{(\xv_i, y_i, m_i)\,:\, m_i\neq y_i\}$ and $(a)$ holds because of sub-linearity of supremum, $(b)$ holds by sub-linearity of supremum, since for two events $E_1$ and $E_2$ we have $\bI_{E_1}\cdot\bI_{E_2}\leq \bI_{E_1}+\bI_{E_2}$, and using Lemma 3.4 of \cite{mohri2018foundations}, and $(c)$ is followed by $\sigma_i$ being zero-mean.
	
	Now, we should bound $A$. Since $u=\sum_{i=1}^n \bI_{m_i\neq y_i}$ is a random variable with distribution $\mathrm{Binomial}(n, \Pr(M\neq Y))$ and using Hoeffding's inequality, we have
	\ea{
	\Pr\big(\tfrac{u}{n}<\Pr(M\neq Y)-t\big)\leq e^{-2nt^2}.
	}
	
	Next, by decomposing $A$, we have
	\ea{
	A &= \Pr\big(\tfrac{u}{n}<\Pr(M\neq Y)-t\big)\Ebb_{\{(\xv_i, y_i, m_i)\}_{i=1}^n}\big[\tfrac{u}{n}\hat{\Rad}_{\Scal}(\Rcal)|\, \tfrac{u}{n}<\Pr(M\neq Y)-t\big]\nonumber\\&~~+\Pr\big(\tfrac{u}{n}\geq\Pr(M\neq Y)-t\big)\Ebb_{\{(\xv_i, y_i, m_i)\}_{i=1}^n}\big[\tfrac{u}{n}\hat{\Rad}_{\Scal}(\Rcal)|\, \tfrac{u}{n}\geq\Pr(M\neq Y)-t\big]\\
	&\leq |\Pr(M\neq Y)-t|e^{-2nt^2}+\min\{\Pr(M\neq Y), \Rad_{n(\Pr(M\neq Y)-t)}(\Rcal)\},
	}
	where the inequality holds since Rademacher complexity is bounded by $1$ and is non-increasing in terms of sample-space size, followed by $\frac{u}{n}\leq 1$, and by means of Lemma 3.4 of \cite{mohri2018foundations}. As a result, by setting $t=\frac{\Pr(M\neq Y)}{2}$, we have
	\ea{
	A\leq \frac{\Pr(M\neq Y)}{2}e^{-\tfrac{n\Pr^2(M\neq Y)}{2}}+\min\{\Pr(M\neq Y), \Rad_{n\Pr(M\neq Y)/2}(\Rcal)\}. \label{eqn: bound_A}
	}
	
Finally using \eqref{eqn: gen_without_j}, \eqref{eqn: rad_j}, and \eqref{eqn: bound_A} we complete the proof.
\end{proof}

\section{Proof of Proposition \ref{prop: consistency}} \label{app: consistency}

To prove the consistency of the deferral surrogate, we know that since $\ell_{\phi}$ is consistent, for every $\{p_1, \ldots, p_{k+1}\}$, such that $\sum_{i=1}^{k+1} p_i =1$, we have
\ea{
\argmax_{i\in [k+1]} \argmin_{h\in \Dcal} \sum_{i=1}^{k+1} p_i \tilde{\ell}_{\phi}(i, h) = \argmax_{i\in [k+1]} p_i. \label{eqn: fisher}
}
(One could prove this by setting $\Pr(X=x)=\delta[x]$, and $\Pr(Y=y|X=x)=p_y$. ) 

Next, we find the minimizer of the loss $\tilde{\ell}_{\phi}$ as 
\ea{
\argmin_{h\in \Fcal} \Ebb_{X, Y, M} \big[\tilde{\ell}_{\phi}(\cv, h)\big] &= \argmin_{h(x)} \Ebb_{Y, M|X=x}\big[\tilde{\ell}_{\phi}(\cv, h(x))\big]\\
&= \argmin_{h(x)} \sum_{i=1}^{k+1} \Ebb[\max_{j\in [k+1]} c(j)-c(i)|X=x] \tilde{\ell}_{\phi}(i, h(x)). \label{eqn: min_surr_def}
}
Next, we form the probability mass function $\{q_1, \ldots, q_{k+1}\}$ as 
\ea{
q_i = \frac{\Ebb\big[\max_{j\in [k+1]} c(j)-c(i)\big]}{\sum_{t=1}^{k+1} \Ebb\big[\max_{j\in [k+1]} c(j)-c(t)\big]}. \label{eqn: def_qi}
}
One could see that the optimizer in \eqref{eqn: min_surr_def} is equivalent to
\ea{
\argmin_{h(x)} \sum_{i=1}^{k+1} q_i \ell_{\phi}\big(i, h(x)\big).  \label{eqn: defer_min_loss}
}
Now, using \eqref{eqn: fisher} and \eqref{eqn: defer_min_loss}, we can show that
\ea{
\argmax_{i\in [k+1]} \argmin_{h\in \Fcal} \Ebb_{X, Y, M} \big[\tilde{\ell}_{\phi}(\cv, h)\big]=\argmax_{i\in [k+1]} q_i = \argmin_{i\in [k+1]} \Ebb[c(i)|X=x].
}
The above identity means that $h_{k+1}(x)\geq \max_{i\in [k]} h_i(x)$ (i.e., $r(x)=1$) iff. we have $\Ebb[c(k+1)|X=x]\leq \min_{i\in [k]}\Ebb[c(i)|X=x]$. Further, we have
\ea{
h(x)=\argmax_{i\in [k]} h_i(x) = \argmin_{i\in [k]}\Ebb[c(i)|X=x].
}

Recalling Proposition 1 in \cite{mozannar2020consistent}, one sees that $r(x)$ and $h(x)$ are that of Bayesian optimal classifier, which proves that $\tilde{\ell}_{\phi}$ is Fisher consistent.

\section {Proof of Theorem \ref{thm: calib_CE}} \label{app: proof_calib_CE}
To show the result for the calibration function, by setting $\Pr(X=x)=\delta[x']$, and $\Pr(Y=y|X=x')=p_y$ for $y\in[k+1]$, we see that
\ea{
L^{0-1}(h)-\min_{h\in \Fcal} L^{0-1}(h)&=\sum_{i\neq h(x')}p_i-\sum_{i\neq \argmax p_i} p_i\\
&=\max_{i\in [k+1]} p_i-p_{h(x')}.
}
Furthermore, we have
\ea{
\tilde{L}_{\phi}(h)-\min_{h\in \Fcal} \tilde{L}_{\phi}(h)=\sum_{i=1}^{k+1}p_i \big[\tilde{\ell}_{\phi}(i, h(x'))-\tilde{\ell}_{\phi}(i, h^*)\big].
}
Hence, $\psi$ being a calibration function proves that
\ea{
\psi(\max p_i-p_{h(x')})\leq \sum_{i=1}^{k+1}p_i \big[\tilde{\ell}_{\phi}(i, h(x'))-\tilde{\ell}_{\phi}(i, h^*)\big], \label{eqn: what_is_calibration}
}
for every choice of $h(x')$.

On the other hand, one could calculate the conditional cost-sensitive loss as 
\ea{
L_{\cv, x}(h) = \Ebb_{Y|X=x}\big[c(h(X))\big] = \sum_{i\neq h(x)} \Ebb_{Y|X=x}\big[c(i)|X=x\big].
}
Hence, we have
\ea{
L_{c, x}(h)-L_{c, x}(h^*)=\Ebb\big[c(h(X))|X=x\big]-\min_{i\in[k+1]}\Ebb[c(i)|X=x\big],
}
where $h^* = \argmin_{h\in \Fcal} L_{c}(h)$.

By defining $q_i$s as \eqref{eqn: def_qi}, one can prove that
\ea{
L_{c, x}(h) - L_{c, x}(h^*) = \sum_{i=1}^{k+1}\big[\max_{j\in[k+1]} c(j)-c(i)|X=x\big] \big(\max q_i - q_{h(x)}\big).
}
For the new surrogate, we further know that
\ea{
\tilde{L}_{\cv, x} (h) &=\Ebb_{Y|X=x}\big[\tilde{\ell}(\cv, h(x))\big]\\
&=\sum_{i=1}^{k+1}\Ebb\big[\max_{j\in [k+1]} c(j)-c(i)|X=x\big]\sum_{i=1}^{k+1} q_i \tilde{\ell}_{\phi}(i, h(x)).
}
Furthermore, one could show that
\ea{
\tilde{h}_{1}^{k+1} &= \argmin_{h\in \Fcal} \tilde{L}_{\cv, x}(h) \\
&= \argmin_{h\in \Fcal} \sum_{i=1}^{k+1} q_i \tilde{\ell}_{\phi}(i, h), \label{eqn: def_htilde}
}
and consequently,
\ea{
\tilde{L}_{\cv, x}(h)-\tilde{L}_{\cv, x}(\tilde{h}_{1}^{k+1}) = \sum_{i=1}^{k+1}\Ebb\big[\max_{j\in [k+1]} c(j)-c(i)|X=x\big] \cdot \sum_{i=1}^{k+1} q_i (\tilde{\ell}_{\phi}(i, h)-\tilde{\ell}_{\phi}(i, \tilde{h}_{1}^{k+1}).
}
Hence, using \eqref{eqn: what_is_calibration} and \eqref{eqn: def_htilde}, we have
\ea{
\Ebb\big[\max_{j\in [k+1]} c(j)-c(i)|X=x\big] \psi(\max_{i\in [k+1]} q_i - q_{h(x)})\leq \tilde{L}_{\cv, x}(h) - \tilde{L}_{\cv, x}( \tilde{h}_{1}^{k+1}).
}
Hence, since $\psi(x)=C|x|^{\ep}$, we have
\ea{
\psi(L_{c, x}(h)-L_{c, x}(h^*)) &= \psi(\Ebb\big[\max_{j\in [k+1]} c(j)-c(i)|X=x\big](\max_{i\in [k+1]}q_i - q_{h(x)})\\&\leq \Ebb^{\ep-1}\big[\max_{j\in [k+1]} c(j)-c(i)|X=x\big](\tilde{L}_{\cv, x}(h)-\tilde{L}_{\cv, x}(\tilde{h}_1^{k+1}))\\
&\overset{(a)}{\leq} M^{\ep-1} (\tilde{L}_{\cv, x}(h)-\tilde{L}_{\cv, x}(\tilde{h}_1^{k+1})), \label{eqn: first_calibration_bound}
}
where $(a)$ holds using the assumption of the theorem.

Finally, using convexity of $\psi$ and by Jensen's inequality, we have
\ea{
\psi(L_c(h)-L_c(h^*)) &= \psi(\Ebb_{X}[L_{c, x}(h)-L_{c, x}(h^*)]\\
&\leq \Ebb_{X} \psi(L_{c, x}(h)-L_{c, x}(h^*))\\
&\overset{(a)}{\leq} M^{\ep-1} \Ebb_{X}\big[\tilde{L}_{\cv, x}(h)-\tilde{L}_{\cv, x}(\tilde{h}_{1}^{k+1})\big]\\
&= M^{\ep-1} (\tilde{L}_{\cv, x}(h)-\tilde{L}_{\cv, x}(\tilde{h}_{1}^{k+1}),
}
in which $(a)$ is followed by \eqref{eqn: first_calibration_bound}. This completes the proof of the first part of theorem.

To obtain the calibration function of the cross-entropy error, we first introduce the following lemma.

\begin{lemma}\label{lem: MAP_KL}
	For every two distributions $P$ and $G$, we have
	\ea{
	\big|\max_i P_i-\max_{i} G_i\big|\leq \sqrt{2 D_{KL}(P\| G)}.
	}
\end{lemma}

\begin{proof}
We define $\argmax_{i} G_i = i_{G}^*$, and $\argmax_{i} P_i = i_{P}^{*}$. If we have $\max_{i}G_i=G_{i_G^*}\geq G_{i_P^*}=\max_{i}P_i$, then 
\ea{
0\leq G_{i_G^*}-P_{i_P^*} \overset{(a)}{\leq} G_{i_G^*}-P_{i_G^*} \overset{(b)}{\leq } \sqrt{2 D_{KL}(P\| G)},
} 
where $(a)$ is correct due to the fact that $\max_{i} P_i \geq P_{i_G^*}$, and $(b)$ holds due to Pinsker's inequality. Further, if we have $\max_{i}G_i=P_{i_G^*}\leq P_{i_P^*}=\max_{i}P_i$, using a similar argument, we have
\ea{
0\leq P_{i_P^*}-G_{i_G^*} \leq P_{i_P^*}-G_{i_P^*} \leq \sqrt{2D_{KL}(P\| G)}.
}
\end{proof}

	Next, we note that the conditional surrogate risk can be rewritten as 
	\ea{
	\tilde{L}_{CE, x}(g_1, \ldots, g_{K+1})&= -\sum_{i=1}^{K+1} \Ebb\big[\max_{j\in [K+1]} c(j) - c(i)|X=x\big]  \log \frac{\exp(g_i(x))}{\sum_k \exp(g_k(x))}\\
	&= N_x H_{\Pcal_x}(\Gcal_x), \label{eqn: rel_ent}
	}
	where $N_x = \sum_{i=1}^{K+1} \Ebb\big[\max_{j\in [K+1]} c(j) - c(i)|X=x\big]$, $H_{\Pcal_x}(\Gcal_x)$ refers to the relative entropy of the distribution $\Gcal_x$ w.r.t $\Pcal_x$ which are defined as
	\ea{
	\Pcal_{x, i} = \frac{\Ebb\big[\max_{j\in [K+1]} c(j) - c(i)|X=x\big]}{N_x},
	} 
	and
	\ea{
	\Gcal_{x, i}= \frac{\exp(g_i(x))}{\sum_k \exp(g_k(x))}.
	}

	Secondly, one note that since in the minimizer of surrogate risk
	$$\argmin_{\gv\in \Fcal} \tilde{L}_{CE} (\gv),$$
	$\Fcal$ contains every function, hence there is no dependency between different point $x$s, and as a result, the minimization is equivalent to finding minimize every conditional surrogate risk. More formally, if $g_1^*, \ldots, g_{K+1}^*$ are such pair of minimizers, we have
	\ea{
	\big(g_1^*(x), \ldots, g_{K+1}^*(x)\big) &= \argmin_{g_1(x), \ldots, g_{K+1}(x)} \tilde{L}_{CE, x}\big(g_1(x), \ldots, g_{K+1}(x)\big)\\
	&\overset{(a)}{=}	\argmin_{g_1(x), \ldots, g_{K+1}(x)} N_x H_{\Pcal_x}(\Gcal_x)\\
	&= \argmin_{g_1(x), \ldots, g_{K+1}(x)} H_{\Pcal_x}(\Gcal_x)\\
	&\overset{(b)}{=} \big(\Pcal_{x, 1}, \ldots, \Pcal_{x, K+1}\big),
	}  
where $(a)$ holds because of \eqref{eqn: rel_ent}, and $(b)$ is a property of relative entropy. 

As a result, the conditional excess surrogate risk can be rewritten as 

\ea{
\tilde{L}_{CE, x}(g_{1}, \ldots, g_{K+1}) -\tilde{L}_{CE, x}^* = N_x H_{\Pcal_x}(\Gcal_x) - N_x H_{\Pcal_x}(\Pcal_x) = N_x D_{KL}(\Pcal_x, \Gcal_x). \label{eqn: ltilde_error}
}

Further, we can write the conditional excess risk as 
\ea{
L_{x}^{0 -1}(g_1, \ldots, g_{K+1}) - &L_{x}^{0 -1}(g^*_1, \ldots, g_{K+1}^*) \nonumber\\&= \Ebb\big[ c(\argmax_{i(x)} g_{i(x)}(x)|X=x]- \min_{i(x)} \Ebb\Big[c\big(i(x)\big)|X=x\Big],
}
where $L_x^{0 - 1}$ is defined as
\ea{
L_x^{0 - 1}(g_1, \ldots, g_{K+1}) = \Ebb\big[\bI_{Y\neq \argmax_{i\in [K+1] g_i(X)}}|X=x\big]
}

	Next, we can rewrite this conditional excess risk in terms of $\Pcal_{x, i}$s as
	\ea{
	L_{x}^{0 -1}(g_1, \ldots, g_{K+1}) - L_{x}^{0 -1}(g^*_1, \ldots, g_{K+1}^*) &= N_x \big(\max_{i(x)} \Pcal_{x, i(x)} - \Pcal_{x, \argmax_{i(x)} g_{i(x)}(x)}\big)\\
	&= N_x\big(\max_{i(x)} \Pcal_{x, i(x)} - \Pcal_{x, \argmax_{i(x)} \Gcal_{x, i(x)}}\big).
}

To bound such a value, we use Pinsker's inequality which states that for every two distributions $P$ and $G$ supported on $\Nbb$, we have
\ea{
TV(P, Q) =\frac{1}{2} \sum_{i} |P_i-Q_i| \leq \sqrt{\frac{D_{KL}(P\| Q)}{2}}.
}

To make use of that inequality, by defining $i_{\Pcal_x} := \argmax_{i(x)} \Pcal_{x, i(x)}$ and $i_{\Gcal_x} ;= \argmax_{i(x)} \Gcal_{x, i(x)}$ and using triangle inequality, we know that
\ea{
N_x \big| \max_{i(x)} \Pcal_{x, i(x)} - \Pcal_{x, \argmax_{i(x)} \Gcal_{x, i(x)}}\big| \leq N_x \big| \Pcal_{x, i_{\Pcal_x}} -  \Gcal_{x, i_{\Gcal_x}}\big| +N_x \big| \Gcal_{x, i_{\Gcal_x}} - \Pcal_{x, i_{\Gcal_x}}\big|.
}

Next, we bound each of these terms separately. Firstly, we know that
\ea{
N_x \big| \Gcal_{x, i_{\Gcal_x}} -  \Pcal_{x, i_{\Gcal_x}}\big| &\leq N_x \sum_{i} |\Pcal_{x, i}-\Gcal_{x, i}|=N_x TV(\Pcal_x \| \Gcal_x)\\&\leq N_x \sqrt{{2D_{KL}(\Pcal_x\| \Gcal_x)}}.
}

Further, using Lemma \ref{lem: MAP_KL}, one can show that 
\ea{
N_x \big| \Pcal_{x, i_{\Pcal_x}} -  \Gcal_{x, i_{\Gcal_x}}\big| \leq N_x \sqrt{{2D_{KL}(\Pcal_x\| \Gcal_x)}}.
}
As a result, we have 
\ea{
L_{x}^{0 -1}(g_1, \ldots, g_{K+1}) - L_{x}^{0 -1}(g^*_1, \ldots, g_{K+1}^*)&\leq N_x \sqrt{{8D_{KL}(\Pcal_x\| \Gcal_x)}}\\
&=\sqrt{8 N_x} \sqrt{\tilde{L}_{CE, x}(g_{1}, \ldots, g_{K+1}) -\tilde{L}_{CE, x}^* },
}
where the last equality is followed by \eqref{eqn: ltilde_error}.
Next, using the upper-bound on $c(i)$s, we have $N_x \leq 2K M$. As a result, we have
\ea{\frac{\big(L_{x}^{0 -1}(g_1, \ldots, g_{K+1}) - L_{x}^{0 -1}(g^*_1, \ldots, g_{K+1}^*)\big)^2}{16MK}\leq \tilde{L}_{CE, x}(g_{1}, \ldots, g_{K+1}) -\tilde{L}_{CE, x}^*.} 
Finally, using Jensen's inequality, we have 
\ea{
\frac{\big(L^{0-1}(g_1, \ldots, g_{K+1}) - L^{0-1}(g^*_1, \ldots, g_{K+1}^*)\big)^2}{16 MK} \leq \tilde{L}_{CE}(g_{1}, \ldots, g_{K+1}) -\tilde{L}_{CE}^*,
}
which yields the statement of theorem.

\section{Proof of Theorem \ref{thm: gen_cross}} \label{app: gen_CE}
We first introduce some useful lemmas, then we get back to the proof of theorem.
\begin{lemma} \label{lem: rad_log_exp}
Let $\Fcal_1, \ldots, \Fcal_k$ be hypothesis classes with Rademacher complexity $\hat{\Rad}_{\Scal}(\Fcal_1), \ldots, \hat{\Rad}_{\Scal}(\Fcal_k)$ on set $\Scal$. The Rademacher complexity of the hypothesis class $\Gcal = \{\log \sum_{i=1}^k e^{f_i(x)}:~ f_i(\cdot)\in \Fcal_i\}$ on set $\Scal$ is bounded as
\ea{
\hat{\Rad}_{\Scal}(\Gcal)\leq \sum_{i=1}^{k} \hat{\Rad}_{\Scal}(\Fcal_i).
}
\end{lemma}
\begin{proof}
We prove this lemma for $k=2$. By following similar steps, one could generalize this proof for every $k\in \Nbb$.

We write the Rademacher complexity of $\Gcal$ as
\ea{
\hat{\Rad}_{\Scal}(\Gcal) &= \frac{1}{m}\Ebb_{\siv}\big[\sup_{f_1\in \Fcal_1, f_2\in \Fcal_2} \sum_{i=1}^m \sigma_i\log (e^{f_1(x)}+e^{f_2(x)}) \big]\\
& = \frac{1}{m} \Ebb_{\siv} \big[\sup_{f_1\in \Fcal_1, f_2\in \Fcal_2}\sum_{i=1}^m \frac{\sigma_i}{2}f_1+\sum_{i=1}^m\frac{\sigma_i}{2}f_2+\sum_{i=1}^m \sigma_i \log (e^{f_1(x)/2-f_2(x)/2}+e^{f_2(x)/2-f_1(x)/2})\big]\\
&\overset{(a)}{\leq} \frac{1}{2m}\Ebb_{\siv}\big[\sup_{f_1\in \Fcal_1}\sum_{i=1}^m \sigma_if_1\big]+\frac{1}{2m}\Ebb_{\siv}\big[\sup_{f_2\in \Fcal_2}\sum_{i=1}^m \sigma_if_2\big] \nonumber\\&\qquad+\frac{1}{m} \Ebb_{\siv} \big[\sum_{i=1}^m \sigma_i \log (e^{f_1(x)/2-f_2(x)/2}+e^{f_2(x)/2-f_1(x)/2})]\\
& = \frac{1}{2}\hat{\Rad}_{\Scal}(\Fcal_1)+\frac{1}{2}\hat{\Rad}_{\Scal}(\Fcal_2) +\hat{\Rad}_{\Scal}\big(\Phi o (\Fcal_1-\Fcal_2)\big), \label{eqn: dec_Rad}
}
where $(a)$ is followed by the sublinearity of supremum, and $\Phi(\cdot)$ is defined as $\Phi(x)=\log (e^{x/2}+e^{-x/2})$.

One could see that $\frac{\partial \Phi(x)}{\partial x} = \frac{1}{2}\frac{e^{x/2}-e^{-x/2}}{e^{x/2}+e^{-x/2}}\leq \frac{1}{2}$, that leads to $\frac{1}{2}$-Lipschitzness of $\Phi(\cdot)$. Using this, and by Ledoux-Talagrand theorem \cite{ledoux1991probability}, we have
\ea{
\hat{\Rad}_{\Scal} \big(\Phi o (\Fcal_1-\Fcal_2)\big)&\leq \frac{1}{2}\hat{\Rad}_{\Scal} (\Fcal_1-\Fcal_2)\\
&= \frac{1}{2m}\Ebb_{\siv}\big[\sup_{f_1\in \Fcal, f_2\in \Fcal_2} \sum_{i=1}^m \sigma_i \big(f_1(x)-f_2(x)\big)\big]\\
&\overset{a}{\leq} \frac{1}{2m}\Ebb_{\siv} \big[\sup_{f_1\in \Fcal_1} \sum_{i=1}^m \sigma_i f_1(x)\big]+\frac{1}{2m}\Ebb_{\siv} \big[\sup_{f_2\in \Fcal_2} \sum_{i=1}^m \sigma_i f_2(x)\big]\\
&= \frac{1}{2}\big(\hat{\Rad}_{\Scal}(\Fcal_1)+\hat{\Rad}_{\Scal}(\Fcal_2)\big), \label{eqn: phi_o_f1f2}
}
where $(a)$ is again followed by sublinearity of supremum.

Finally, using \eqref{eqn: dec_Rad} and \eqref{eqn: phi_o_f1f2}, we complete the proof.
\end{proof}

\begin{lemma} \label{lem: softmax}
Let $\Fcal$ be a hypothesis class of functions $f(x, y): \Xcal\times [k+1]\to \Rbb$, and $\Pi_1(\Fcal)=\{x\to f(x, y):\,f(\cdot, \cdot)\in \Fcal, y\in [k+1]\}$. Then,
\begin{itemize}
    \item for $\Gcal = \{x, y\to f(x, y)-\log \sum_{j=1}^{k+1} f(x, y):\, f(\cdot, \cdot)\in \Fcal\}$  and given the assumption that for every label inside sets of pairs $(x_i, y_i)\in \Scal$ is within the range $\{1, \ldots, k\}$, we have
    \ea{
    \hat{\Rad}_{\Scal} (\Gcal)\leq (2k+1)\hat{\Rad}_{\Scal_x}\big(\Pi_1(\Fcal)\big),
    }
    \item and for $\Hcal_i = \{x\to f(x, i) - \log \sum_{y=1}^{k+1} f(x, y):\, f(\cdot, \cdot)\in \Fcal\}$, we have
    \ea{
    \hat{\Rad}_{\Scal}(\Hcal_i)\leq (k+2)\hat{\Rad}_{\Scal_x} \big(\Pi_1(\Fcal)\big).
    }
\end{itemize}
\end{lemma}
\begin{proof}
\begin{enumerate}
    \item We write Rademacher complexity of $\Gcal$ as
\ea{
\hat{\Rad}_{\Scal} (\Gcal) &= \frac{1}{m}\Ebb_{\siv} \big[\sup_{f\in \Fcal} \sum_{i=1}^{m} \sigma_i f(x_i, y_i)-\sigma_i \log \sum_{y=1}^{k+1} e^{f(x_i, y)}\big]\\
&\overset{(a)}{\leq} \underbrace{\frac{1}{m}\Ebb_{\siv}\big[\sup_{f\in \Fcal} \sigma_i f(x_i, y_i)\big]}_{A} +\underbrace{\frac{1}{m}\Ebb_{\siv}\big[\sup_{f\in \Fcal} \sum_{i=1}^m \sigma_i \log \sum_{y=1}^{k+1} e^{f(x_i, y)}\big]}_{B}, \label{eqn: dec_Rad_G}
}
where $(a)$ holds because of sublinearity of supremum.
Next, we bound $A$ and $B$ as follows.

First, we know that
\ea{
A &= \frac{1}{m} \Ebb_{\siv} \big[\sup_{f\in \Fcal} \sum_{y=1}^{k} \sum_{i=1}^m \sigma_i f(x_i, y)\bI_{y_i = y}\big]\\
&\leq \sum_{y=1}^{k} \frac{1}{m} \Ebb_{\siv} \big[\sup_{f\in \Fcal} \sum_{i=1}^m \sigma_i f(x_i, y) (\ep_i/2+1/2)\big],
}
where $\ep_i = 2\bI_{y_i = y}-1$. Hence, again, applying sublinearity of supremum, we have
\ea{
A&\leq \sum_{y=1}^{k} \frac{1}{m}\Ebb_{\siv}\big[\sup_{f\in \Fcal} \sum_{i=1}^m \frac{\sigma_i \ep_i}{2} f(x_i, y)\big]+\frac{1}{m}\Ebb_{\siv}\big[\sup_{f\in \Fcal} \sum_{i=1}^m \frac{\sigma_i }{2} f(x_i, y)\big]. \label{eqn: bound_A_2}
}
Since $\ep \in \{-1, 1\}$, then $\sigma_i \ep_i$ take Rademacher distribution as well. Hence, using \eqref{eqn: bound_A_2}, we have 
\ea{
A &\leq \sum_{y=1}^{k} \frac{1}{2}\hat{\Rad}_{\Scal_x} \big(\Pi_1(\Fcal)\big)+ \frac{1}{2}\hat{\Rad}_{\Scal_x} \big(\Pi_1(\Fcal)\big)\\
& = k \hat{\Rad}_{\Scal_x} \big(\Pi_1(\Fcal)\big). \label{eqn: final_bound_A}
}

Next, to bound $B$, using Lemma \ref{lem: rad_log_exp}, we have
\ea{
B \leq \sum_{y=1}^{k+1} \frac{1}{m} \Ebb_{\siv} \big[\sup_{f\in \Fcal} \sum_{i=1}^m \sigma_i f(x_i, y)\big] \leq \hat{\Rad}_{\Scal_x} \big(\Pi_1(\Fcal)\big). \label{eqn: bound_B}
}
Finally, using \eqref{eqn: dec_Rad_G}, \eqref{eqn: final_bound_A}, and \eqref{eqn: bound_B}, we complete the proof.
\item We bound Rademacher complexity of $\Hcal_i$ as
\ea{
\hat{\Rad}_{\Scal_x}(\Hcal_i) & =\frac{1}{m} \Ebb_{\siv} \big[\sup_{f\in \Fcal} \sum_{j=1}^m \sigma_j f(x_j, i)-\sigma_j \log \sum_{y=1}^{k+1} e^{f(x_j, i)}\big]\\
&\overset{(a)}{\leq} \frac{1}{m} \Ebb\big[\sup_{f\in \Fcal} \sum_{j=1}^m \sigma_j f(x_j, i)\big]+\frac{1}{m} \Ebb_{\siv} \big[\sup_{f\in \Fcal} \sum_{j=1}^m \sigma_j \log \sum_{y=1}^{k+1} e^{f(x_j, y)}\big]\\
&\overset{(b)}{\leq} \hat{\Rad}_{\Scal_x} \big(\Pi_1(\Fcal)\big) +\sum_{y=1}^{k+1}\frac{1}{m} \Ebb_{\siv} \big[\sup_{f\in \Fcal} \sum_{j=1}^m f(x_j, y)\big]\\
&\leq (k+2) \hat{\Rad}_{\Scal_x} \big(\Pi_1(\Fcal)\big),
}
where $(a)$ is followed by sublinearity of supermum, and $(b)$ because of Lemma \ref{lem: rad_log_exp} and using definition of $\Pi_1(\Fcal)$. 

\end{enumerate}
\end{proof}

\begin{lemma}\label{lem: Rad_loss}
For $i\in \{1, \ldots, k+1\}$ let $\Hcal_i$ be hypothesis class of functions $h_i(x):\Xcal\to \Rbb$ with bounded norm $\|h_i\|_{\infty}<C$. Further, let $\Pi_1(\Hcal) = \{x\to h_i(x):\, h_i\in \Hcal_i, i\in [k+1]\}$. The Rademacher complexity of the class $\Lcal$ of loss functions 
\ea{
\ell(x, y, m) = - \log \frac{e^{-h_y(x)}}{\sum_{i=1}^{k+1} e^{-h_i(x)}} - \bI_{m\neq y} \log \frac{e^{-h_{k+1}(x)}}{\sum_{i=1}^{k+1} e^{-h_i(x)}},
}
for $m, y\in [k]$ is bounded as
\ea{
\Rad_{n}(\Lcal) &\leq (k+1) \Rad_{n}\big(\Pi_1(\Hcal)\big)+(k+2)\min\{\Pr(M\neq Y), \Rad_{n\Pr(M\neq Y)/2}\big(\Pi_1(\Hcal)\big)\}\nonumber\\&\quad+\frac{C}{2}\Pr(M\neq Y)(k+2)e^{-n\Pr^2(M\neq Y)/2}.
}
\end{lemma}

\begin{proof}
We write empirical Rademacher complexity of $\Lcal$ as
\ea{
\hat{\Rad}_{\Scal}(\Lcal) &= \frac{1}{n}\Ebb_{\siv} \big[\sup_{\ell\in \Lcal} \sum_{i=1}^n \sigma_i \ell(x_i, y_i, m_i)\big]\\
&=\frac{1}{n}\Ebb_{\siv}\big[\sup_{h_j\in \Hcal_j} \sum_{i=1}^n \sigma_i (-\log \frac{e^{-h_y(x_i)}}{\sum_{j=1}^{k+1} e^{-h_j(x_i)}}-\bI_{m\neq y}\log \frac{e^{-h_{k+1}(x_i)}}{\sum_{j=1}^{k+1} e^{-h_j(x_i)}})\big]\\
&\leq \frac{1}{n}\Ebb_{\siv}\big[\sup_{h_j\in \Hcal_j} \sum_{i=1}^n \sigma_i\log \frac{e^{-h_y(x_i)}}{\sum_{j=1}^{k+1} e^{-h_j(x_i)}}\big]+\frac{1}{n}\Ebb_{\siv}\big[\sup_{h_j\in \Hcal_j} \sum_{i=1, y_i \neq m_i}^n \sigma_i\log \frac{e^{-h_{k+1}(x_i)}}{\sum_{j=1}^{k+1} e^{-h_j(x_i)}}\big]\\
&\overset{(a)}{\leq} (k+1) \hat{\Rad}_{\Scal_x} \big(\Pi_1(\Hcal)\big) +(k+2)\frac{\sum_{i=1}^n \bI_{m_i \neq y_i}}{n}\hat{R}_{\Scal_x|m_i\neq y_i} \big(\Pi_(\Hcal)\big), \label{eqn: bound_rad_loss}
}
where $(a)$ holds by applying Lemma \ref{lem: softmax}. 

Using \ref{eqn: bound_rad_loss}, and by calculating the expectation over $\{(x_i, y_i, m_i)\}_{i=1}^n$, we have
\ea{
\Rad_n(\Lcal)\leq (k+1)\Rad_n \big(\Pi_1(\Hcal)\big)+(k+2)\underbrace{\Ebb_{\{(x_i, y_i, m_i)\}_{i=1}^n}\Big[\frac{\sum_{i=1}^n \bI_{m_i\neq y_i}}{n} \hat{\Rad}_{\Scal_x|m_i\neq y_i}\big(\Pi_1(\Hcal)\big)\Big]}_{A}. \label{eqn: bound_rad_exp_loss}
}

It is remained to bound $A$. For this task, we first notice that $u =\sum_{i=1}^n \bI_{m_i\neq y_i}$ is a random variable with distribution $\text{Binomial}\big(n, \Pr(M\neq Y)\big)$. Further, by Hoeffding's inequality we know that for $t>0$, we have 
\ea{
\Pr(\frac{u}{n}<\Pr(M\neq y)-t)\leq e^{-2nt^2}.
}
Hence, by decomposing $A$, we have
\ea{
A &= \Pr(\frac{u}{n} < \Pr(M\neq Y)-t) \Ebb\big[\frac{u}{n}\hat{\Rad}_{\Scal_x|y_i\neq m_i}|\frac{u}{n}<\Pr(M\neq Y)-t\big]\nonumber\\
&\quad +\Pr(\frac{u}{n} \geq \Pr(M\neq Y)-t) \Ebb\big[\frac{u}{n}\hat{\Rad}_{\Scal_x|y_i\neq m_i}|\frac{u}{n}\geq\Pr(M\neq Y)-t\big]\\
&\leq C|\Pr(M\neq Y)-t|e^{-2nt^2} +\min\{\Pr(M\neq Y), \Rad_{n(\Pr(M\neq Y)-t)}\big(\Pi_1(\Hcal)\big)\},\label{eqn: bound_A_chernoff}
}
where the last inequality holds because (1) every function in $\Pi_1(\Hcal)$ is bound by $C$, and so is the Rademacher complexity of $\Pi_1(\Hcal)$, and (2) the Rademacher complexity is non-increasing with the sample space size. 

Hence, using \ref{eqn: bound_rad_exp_loss}, \ref{eqn: bound_A_chernoff}, and by setting $t=\frac{\Pr(M\neq Y)}{2}$ the proof is complete.
\end{proof}

\begin{proof}[Proof of Theorem \ref{thm: gen_cross}]
Using Rademacher inequality on generalization error (e.g., Theorem 3.3 of \cite{mohri2018foundations}), we know that with probability at least $1-\delta/2$, we have
\ea{
\tilde{L}_{CE}(\fv_1^{k+1}) \leq \hat{L}_{CE}(\fv_1^{k+1}) + 2\Rad_n(\Lcal) + \sqrt{\frac{D\log 2/\delta}{2n}}, \label{eqn: Rademacher_inequality}
}
where $D$ is an a upper-bound on $\|\ell\|_{\infty}$ for $\ell\in \Lcal$, and where $\hat{L}_{CE}$ is the empirical loss corresponding to $\ell_{CE}$, and $f_i\in \Fcal_i$. 

We follow the proof in three steps, (i) we find $D$, (ii) we find a lower-bound on $\tilde{L}_{CE}(\fv_1^{k+1})$ in terms of $\defl(\fv_1^{k+1})$, and (iii) we complete the proof by bounding the difference $|\min_{\fv\in \Fcal_1^{k+1}} \hat{L}_{CE}(\fv_1^{k+1})-\min_{\fv_1^{k+1}\in \Fcal} \tilde{L}_{CE}(\fv_1^{k+1})|$. 

\noindent
	$\bullet$ {\bfseries Step (i)}: For calculating a bound on $\|\ell\|_{\infty}$ for $\ell\in \Lcal$, we use boundedness of $\|f_i\|_{\infty}$ for $i\in [k+1]$ and $\fv_{1}^{k+1}\in \Fcal_{1}^{k+1}$. Indeed, we know the function 
	\ea{
	b_D(x) = -\log \frac{e^{-x}}{e^{-x}+D},
	}
	for $D>0$ is a monotonically non-increasing function of $x$. Hence, over a closed interval, it takes the minimum and maximum on the limit points. As a result, for $|x|\leq C$, we have
	\ea{
	0\leq b_D(x)\leq -\log \frac{e^{-C}}{e^{-C}+D}. \label{eqn: bound_b(x)}
	}
	Hence, for the loss function $\ell(x, y, m)$, in which $\|\fv_1^{k+1}\|\leq C$, we have
	\ea{
	0< \ell(x, y, m) &= b_{\sum_{i=1, i\neq y}^{k+1}e^{-f_i(x)}}\big(f_y(x)\big) +\bI_{m\neq y}b_{\sum_{i=1 }^{k}e^{-f_i(x)}}\big(f_{k+1}(x)\big)\\&\leq -\log \frac{e^{-C}}{e^{-C}+\sum_{i=1, i\neq y}^{k+1}e^{-f_i(x)}}-\bI_{m\neq y} \log \frac{e^{-C}}{e^{-C}+\sum_{i=1 }^{k}e^{-f_i(x)}}\\
	& -2\log \frac{e^{-C}}{e^{-C}+ke^{C}} \leq -2\log \frac{e^{-2C}}{k+1} = 4C - 2\log (k+1). \label{eqn: bound_loss}
	}

\noindent
	$\bullet$ {\bfseries Step (ii)}: Using excessive surrogate risk bound, we see that
	\ea{
	\psi\big(\defl(\fv_1^{k+1})-\min_{\hv_1^{k+1}} \defl(\hv_1^{k+1})\big) +\min_{\hv_1^{k+1}} \tilde{L}_{CE}(\hv_1^{k+1})\leq \tilde{L}_{CE}(\fv_1^{k+1}). \label{eqn: calib}
	}
	
\noindent
	$\bullet$ {\bfseries Step (iii)}: In this step, we find a bound on $\hat{L}_{CE}(\fv_1^{k+1})-\min_{\hv} \tilde{L}_{CE}(\hv_1^{k+1})$. Indeed, we know that
	\ea{
	\hat{L}_{CE}(\fv_1^{k+1}) - \min_{\hv_1^{k+1}} \tilde{L}_{CE}(\hv_1^{k+1}) &= \underbrace{\hat{L}_{CE}(\fv_1^{k+1}) -\min_{\hv_1^{k+1}\in \Fcal_1^{k+1}}\hat{L}_{CE}(\hv_1^{k+1})}_{e_{min}}\nonumber\\&+\min_{\hv_1^{k+1}\in \Fcal_1^{k+1}}\hat{L}_{CE}(\hv_1^{k+1})\quad-\min_{\hv_1^{k+1}\in \Fcal_1^{k+1}}\tilde{L}_{CE}(\hv_1^{k+1})\nonumber\\&+\underbrace{\min_{\hv_1^{k+1}\in \Fcal_1^{k+1}}\tilde{L}_{CE}(\hv_1^{k+1}) - \min_{\hv_1^{k+1}} \tilde{L}_{CE}(\hv_1^{k+1})}_{e_{\phi-\mathrm{appr}}}\\
	&\leq \hat{L}_{CE}(\tilde{\hv}_{1}^{k+1})-\tilde{L}_{CE}(\tilde{\hv}_{1}^{k+1}) + e_{\min}+ e_{\phi-\mathrm{appr}},
	}
	where $\tilde{\hv}_{1}^{k+1} = \argmin_{\hv_1^{k+1}\in \Fcal_1^{k+1}} \tilde{L}_{CE}(\hv_1^{k+1})$. Hence, using Hoeffding's inequality, with probability at least $1-\delta/2$, we have 
	\ea{
	\hat{L}_{CE}(\fv_1^{k+1}) -\min_{\hv_1^{k+1}}\tilde{L}_{CE}(\hv_1^{k+1})\leq \sqrt{\frac{D}{2n}\log {2/\delta}}+e_{\min}+e_{\phi-\mathrm{appr}}. \label{eqn: bound_Hoeffding_Rhat_R}
	}
	
Finally, using Lemma \ref{lem: Rad_loss}, \ref{eqn: Rademacher_inequality}, \ref{eqn: bound_loss}, \ref{eqn: calib}, \ref{eqn: bound_Hoeffding_Rhat_R}, and by union bound, we complete the proof.
\end{proof}

\section{Proof of Proposition \ref{prop: active_1}}\label{app: active_1}

We prove this proposition in four steps: (i) we first prove that in each iteration, the deferral loss $\defl(h,r)$ is bounded, (ii) using Step (i), we show that $\Pr(X\in DIS(V_i))$ halves in each iteration with high probability, (iii) using Step (ii) we conclude that $\Pr(X\in DIS(V_{\lceil \log \tfrac{1}{\ep}\rceil}))\leq \ep$ with high probability, and finally (iv) we provide a bound on $\defl(h, r)$ using the result in Step (iii).

\noindent
	$\bullet$ {\bfseries Step (i)}: We use Theorem 2 of \cite{mozannar2020consistent} that making use of realizability of $(h, r)$ on empirical distribution shows that with probability at least $1-\delta'$ we have
	\ea{
    \Ebb\big[\bI_{r(X)=0}\bI_{h(X)\neq Y}+\bI_{r(X)=1}\bI_{M\neq Y}|X\in DIS(V_i)\big]\leq \sqrt{\tfrac{2\log 2/\delta'}{m_i}} +\sqrt{\tfrac{2d(\Hcal)\log \tfrac{em_i}{d(\Hcal)}}{m_i}}+\sqrt{\tfrac{32d(\Rcal)\log \tfrac{em_i}{d(\Rcal)}}{m_i}}, \label{eqn: loss_hussein}
	}
	where $m_i$ is the size of the set on which human provides the prediction in each iteration. Note that we draw only samples from $DIS(V_i)$, and that is the reason that we condition the loss on $X$ being in $DIS(V_i)$ .
	
	To analyze the sample complexity that corresponds to \eqref{eqn: loss_hussein}, we let $\delta'=\frac{\delta}{(2+\lceil\log \tfrac{1}{\ep}\rceil-i)^2}$ and we assume that \ea{
	m_i\geq \max\{108\Theta^2\log \frac{(2+\lceil \log \tfrac{1}{\ep} \rceil-i)^2}{\delta}, 360\Theta^2 d(\Hcal) \log \Theta, 2, 276\Theta^2 d(\Hcal)\log \Theta\}.\label{eqn: samp_comp}
	}
	Using the first term in RHS of \eqref{eqn: samp_comp}, we bound the first term in the upper-bound \eqref{eqn: loss_hussein} as
	\ea{
	\sqrt{\frac{2\log \frac{2}{\delta'}}{m_i}}\leq \sqrt{\frac{2\log \frac{2(2+\lceil \log \frac{1}{\ep}\rceil-i)^2}{\delta}}{108\Theta^2\log \frac{(2+\lceil \log \frac{1}{\ep}\rceil-i)^2}{\delta}}} = \frac{1}{6\Theta}\sqrt{\frac{2}{3}+\frac{2}{3\log \frac{(2+\lceil \log \frac{1}{\ep}\rceil-i)^2}{\delta}}}.
	}
	Then, for $i\leq \lceil\log\frac{1}{\ep}\rceil$ and since $\delta\leq 1$, we know that $\log \frac{4}{\delta}\geq 2$, which concludes that
	\ea{
	\sqrt{\frac{2\log \frac{2}{\delta'}}{m_i}}\leq \frac{1}{6\Theta}.\label{eqn: first_term_ub}
	}
	Further, using the second and third term in RHS of \eqref{eqn: samp_comp} and since $\sqrt{\frac{2 d(\Hcal)\log em_i}{m_i}}$ is monotonically decreasing for $m_i\geq 2$ (note that $\frac{\partial}{\partial x}\big(\frac{\log x}{x}\big)=\frac{1}{x^2}-\frac{\log x}{x^2}\leq 0$ for $x\geq 2$) we have 
	\ea{
	\sqrt{\frac{2 d(\Hcal)\log em_i}{m_i}}\leq \sqrt{\frac{2 d(\Hcal) \log \frac{e\Theta^2 d(\Hcal) \log \Theta}{d(\Hcal)}}{360\Theta^2 d(\Hcal) \log \Theta}}&= \sqrt{\frac{2\log\big( e\Theta^2\cdot\log\Theta\big)}{360\Theta^2\log\Theta}}\\
	&=\frac{1}{\Theta}\sqrt{\frac{\log e+2\log \Theta+\log\log\Theta}{180\log \Theta}}.
	}
	If we set $\Theta\geq e$, we have $\log \Theta\geq \log e$, and since $\log\log\Theta\leq \log \Theta$ for $\Theta\geq 1$, we have
	\ea{
	\sqrt{\frac{2d(\Hcal) \log \frac{em_i}{d(\Hcal)}}{m_i}}\leq \frac{1}{\Theta} \sqrt{\frac{5\log \Theta}{180\log \Theta}}=\frac{1}{6\Theta}. \label{eqn: second_term_ub}
	}
	Similarly, we could show that
	\ea{
	\sqrt{\frac{32d(\Rcal) \log \frac{em_i}{d(\Rcal)}}{m_i}}\leq \frac{1}{6\Theta},
	}
	which together with \eqref{eqn: loss_hussein}, \eqref{eqn: first_term_ub}, and \eqref{eqn: second_term_ub} proves that for  $m_i=O\big(\Theta^2(d(\Hcal)\log\Theta+d(\Rcal)\log \Theta+\log \tfrac{(2+\lceil \log \tfrac{1}{\ep}\rceil-i)^2}{\delta})\big)$ we have
	\ea{
	\Ebb\big[\bI_{r(X)=0}\bI_{h(X)\neq Y}+\bI_{r(X)=1}\bI_{M\neq Y}|X\in DIS(V_i)\big]\leq \frac{1}{2\Theta},
	}
	with probability at least $1-\frac{\delta}{(2+\lceil \log \tfrac{1}{\ep}\rceil-i)^2}$.
	
	Since $X\in DIS(V_i)$ is a necessary condition for $\bI_{r(X)=0}\bI_{h(X)\neq Y}+\bI_{r(X)=1}\bI_{M\neq Y}=1$, we conclude that
	\ea{
	\defl(h, r)=\De(V_i) \Ebb\big[\bI_{r(X)=0}\bI_{h(X)\neq Y}+\bI_{r(X)=1}\bI_{M\neq Y}|X\in DIS(V_i)\big]\leq \frac{\De(V_i)}{2\Theta},
	}
	with probability at least $1-\frac{\delta}{(2+\lceil \log \tfrac{1}{\ep}\rceil-i)^2}$, where $\De(V_i)$ is defined as
	\ea{
	\De(V_i):=\Pr\big(X\in DIS(V_i)\big)
	}
	
\noindent
	$\bullet$ {\bfseries Step (ii)}: Since $\defl(h, r)=\Pr\big(r(X)M+(1-r(X))h(X)\neq Y\big)$, and because $\defl(h^*, r^*)=0$, and using Step (i), we have
	\ea{
	\Pr\big(r(X)M+(1-r(X))h(X)\neq r^*(X)M+(1-r^*(X))h^*(X)\big)\leq \frac{\De(V_i)}{2\Theta},
	}
	with probability at least $1-\frac{\delta}{(2+\lceil \log \tfrac{1}{\ep}\rceil-i)^2}$. As a result, for all $(h, r)\in V_{i+1}$, we have $(h, r)\in B\big((h^*, r^*), \tfrac{\De(V_i)}{2\Theta}\big)$ with such probability. 
	
	Hence, we have
	\ea{
	\De(V_{i+1})\leq \De\Big(B\big((h^*, r^*), \tfrac{\De(V_i)}{2\Theta}\big)\Big)\leq \Theta\cdot\frac{\De(V_i)}{2\Theta}=\frac{\De(V_i)}{2},
	}
	where the last inequality is followed by the definition of $\Theta$.
	
\noindent
	$\bullet$ {\bfseries Step (iii)}:	Using union bound, and since 
	\ea{
	\sum_{i=1}^{\lceil \log \tfrac{1}{\ep}\rceil} \frac{\delta}{(2+\lceil\log \tfrac{1}{\ep}\rceil-i)^2} = \sum_{i=2}^{\lceil \log \tfrac{1}{\ep}\rceil+1}\frac{\delta}{i^2}\leq \sum_{i=2}^{\infty}\frac{\delta}{i^2}=\frac{\pi^2-6}{6}\cdot\delta\leq \delta,}
	and using Step (iii), we have that
	\ea{
	\De(V_{\lceil\log\tfrac{1}{\ep}\rceil})\leq\frac{1}{2^{\lceil\log \tfrac{1}{\ep}\rceil}}\De(V_0)\leq \ep,
	}
	with probability at least $1-\delta$.
	
\noindent
	$\bullet$ {\bfseries Step (iv)}: Since we know that $\defl(h^*, r^*)=0$, we conclude that \ea{
 \Pr\big(M\neq Y, r^*(X)=1\big)=0.\label{eqn: my_r*_zero}}
    
    Next, since for all $h\in V_{\lceil\log \tfrac{1}{\ep}\rceil}$ we have $\Pr\big(h(X)\neq Y, r(X)=0\big)=0$, we can show that
    \ea{
    \defl(h, r)&=\Pr\big(h(X)\neq Y, r(X)=0\big)+\Pr\big(M\neq Y, r(X)=1\big)\\
    &= \Pr\big(M\neq Y, r(X)=1\big)\\
    &= \Pr(M\neq Y, r(X)=1, r^*(X)=0)+\Pr(M\neq Y, r(X)=1, r^*(X)=1)\\
    &\overset{(a)}{=} \Pr(M\neq Y, r(X)=1, r^*(X)=0)\\
    & = \Pr(M\neq Y, r(X)\neq r^*(X), r^*(X)=0)\leq \Pr\big(r(X)\neq r^*(X)\big),\label{eqn: deflbound}
    }
    where $(a)$ is followed by \eqref{eqn: my_r*_zero}.
    
    Next, since $(h^*, r^*)$ is not removed in any iteration because of its consistency, we have $(h^*, r^*)\in V_{\lceil\log \tfrac{1}{\ep}\rceil}$. Hence using Step (iii), for all $(h, r)\in V_{\lceil\log \tfrac{1}{\ep}\rceil}$ we have
    \ea{
    \Pr\big(r(X)\neq r^*(X)\big)\leq \Pr\big(X\in DIS(V_{\lceil\log \tfrac{1}{\ep}\rceil})\big)\leq \ep, \label{eqn: rr*bound}
    }
    with probability at least $1-\delta$. 
    
    Using \eqref{eqn: deflbound} and \eqref{eqn: rr*bound} the proof is complete.
    
\section{An example on which CAL algorithm fails}\label{app: counter_CAL}

Here, we provide the reader with an example on which vanilla CAL algorithm in Section \ref{subsec:theory_active_learning} does not converge. Let $\Xcal=\{0, 1\}$ and let $X\sim \mathrm{Uniform}\{\Xcal\}$ and $\mu_{XYM}=\mu_{X}\bI_{Y=X}\bI_{M=0}$, which means for all instances on $\Xcal$, $Y=X$ and $M=0$. Further, let $\Hcal=\{h_1, h_2\}$, and $\Rcal=\{r_1, r_2\}$, where 
\ea{
h_1(\xv)=r_1(\xv)=\xv, \, \, h_2(\xv)=r_2(\xv)=0.
}
One could see that in this case three pairs $(h_1, r_1), (h_1, r_2), (h_2, r_1)$ as deferral systems provide zero loss.

To run CAL, we draw a sample from $\mu_{X}$. Assume that we observe $\xv=1$. We see that since $\xv\in DIS(V_0)=\{1\}$, then we need to query human's prediction and true label on such instance. Hence, we collect the corresponding values $y=m=1$ for that instance. Next, we update the version space
\ea{
V_1=\{(h_1, r_1), (h_1, r_2), (h_2, r_1)\}, 
}
to induce consistency. However, we note that $DIS(V_1)$ does not change comparing to $DIS(V_0)$. Hence, $\Pr(X\in DIS(V_0))=\Pr(X\in DIS(V_1))=\ldots=\frac{1}{2}$. As a result, CAL algorithm does not converge, and in each iteration queries human prediction for $\xv=1$.

\section{Proof of Theorem \ref{thm: Dod}}\label{app: active}
Using Theorem 5.1 of \cite{hanneke2014theory}, we know that if $n_l=C\Theta d(\Dcal)\log \big(\frac{4 \Theta}{\delta}\log \frac{4}{\ep}\big)\log \frac{4}{\ep}$, then with probability at least $1-\frac{\delta}{4}$ we have
\ea{
\Pr\big[f(X)\neq \bI_{M\neq Y}\big]\leq \frac{\ep}{4}. \label{eqn: real_error_less_ep}
}

Next, we bound the empirical joint loss on unlabeled samples. We know that
\ea{
\hat{L}_{\mathrm{def}}^{0 - 1}(h, r) &= \frac{1}{n_u}\sum_{i} \bI_{h(x_i)\neq y_i} \bI_{r(x_i)=0}+\bI_{m_i\neq y_i}\bI_{r(x_i)=1}\\
&= \frac{1}{n_u}\sum_{i}\big[\bI_{h(x_i)\neq y_i} \bI_{r(x_i)=0} +f(x_i)\bI_{r(x_i)=1}\big]+\frac{1}{n_u}\sum_{i}\big(\bI_{m_i\neq y_i}-f(x_i)\big)\bI_{r(x_i)=1}\\
&\overset{(a)}{=} \frac{1}{n_u}\sum_{i}\big(\bI_{m_i\neq y_i}-f(x_i)\big)\bI_{r(x_i)=1}\\
&\leq \frac{1}{n_u}\sum_{i}|\bI_{m_i\neq y_i}-f(x_i)|\\
&=\frac{1}{n_u}\sum_{i}\bI_{f(x_i)\neq \bI_{m_i\neq y_i}} \label{eqn: first_conc}
}
where $(a)$ holds because of Line 5 in Algorithm \ref{alg:disagreement}. 

As a result, we use Hoeffding's inequality coupled with \eqref{eqn: real_error_less_ep} to show that
\ea{
\hat{L}_{\mathrm{def}}^{0 - 1}(h, r) \leq \frac{\ep}{4}+\sqrt{\frac{\log 2/\delta}{2n_u}}, \label{eqn: empirical_error_bound}
}
with probability at least $1-\frac{3\delta}{4}$.  Further, by generalization bound in Theorem 2 of \cite{mozannar2020consistent}, with probability at least $1-\frac{\delta}{4}$ we have
\ea{
\defl(h, r)\leq \hat{L}_{\mathrm{def}}^{0-1}(h, r) +\sqrt{\frac{2\log 8/\delta}{n_u}}+\Rad_{n_u}(\Hcal)+4\Rad_{n_u}(\Rcal)+\Pr(M\neq Y)e^{\frac{-n_u \Pr(M\neq Y)}{8}},
}
where following \eqref{eqn: empirical_error_bound} we conclude that with probability at least $1-\delta$ we have
\ea{
\defl(h, r)\leq \frac{\ep}{4}+\sqrt{\frac{\log 2/\delta}{2n_u}}+\sqrt{\frac{2\log 8/\delta}{n_u}}+\Rad_{n_u}(\Hcal)+8\Rad_{n_u}(\Rcal)+\Pr(M\neq Y)e^{\frac{-n\Pr(M\neq Y)}{8}}. \label{eqn: gen_bound_1}
}
One can further calculate an upper-bound on $\Rad_{n_u}(\Hcal)$ and $\Rad_{n_u}(\Rcal)$ using  Corollary 3.8 and 3.18 of \cite{mohri2018foundations} as
\ea{
\Rad_{n_u}(\Hcal)\leq \sqrt{\frac{2 d(\Hcal) \log \frac{en_u}{d(\Hcal)}}{n_u}},
}
and
\ea{
\Rad_{n_u}(\Rcal)\leq \sqrt{\frac{2 d(\Rcal)\log \frac{en_u}{d(\Rcal)}}{n_u}},
}
which by substituting in \eqref{eqn: gen_bound_1} we conclude that
\ea{
\defl(h, r)\leq &\frac{\ep}{4}+\sqrt{\frac{\log \frac{2}{\delta}}{2n_u}}+\sqrt{\frac{2\log \frac{8}{\delta}}{n_u}}+\sqrt{\frac{2d(\Hcal)\log \frac{en_u}{d(\Hcal)}}{n_u}}+\sqrt{\frac{32 d(\Rcal) \log \frac{en_u}{d(\Rcal)}}{n_u}}\nonumber\\&+\Pr(M\neq Y) e^{\frac{-n\Pr(M\neq Y)}{8}}. \label{eqn: last_conc}
}

Finally, using \eqref{eqn: last_conc} and by letting $n_u\geq \max\{\tfrac{8\log \frac{2}{\delta}}{\ep^2}, \tfrac{288 \log 8/\delta}{\ep^2}, \tfrac{C' \max\{d(\Hcal), d(\Rcal)\}\log \frac{1}{\ep}}{\ep^{2}}\}$ in which $C'=2^{10}$ and for $\ep\leq \frac{1}{2^{18}e^4}$, we have
\ea{
\defl(h, r)\leq \ep,
}
with probability at least $1-\delta$, which completes the proof.




	

\section{Experimental Details} \label{app: experiments}

\paragraph{Data.} We use the CIFAR validation set of 10k images as the test set and split the CIFAR training set 90/10 for training and validation. 

\paragraph{Optimization.} We use the AdamW optimizer \cite{loshchilov2017decoupled} with learning rate $0.001$ and default parameters on PyTorch. We also use a cosine annealing learning rate scheduler and train for 100 epochs and saving the best performing model on the validation set.  For the surrogate $L_{CE}^{\alpha}$ \cite{mozannar2020consistent}, we perform a search for $\alpha$ over a grid $[0, 0.1, 0.5,1]$. 

\paragraph{Model Complexity.} For the model complexity gap figure, we use a convolutional neural network consisting of two convolutional layers with a max pooling layer in between followed by three fully connected layers with ReLU activations. We modify respectively:  the number of channels produced by the convolution of the first layer and of the second layer, and the number of units in the first and second fully connected layers. We use this set of parameters to produce the plot for the classifier model:
\begin{align*}
 &[ [1,1,50,25], [3,3,50,25], [4,4,80,40], [6,6,100,50],[12,12,100,50], \\ &[20,20,100,50],[100,100,500,250],[100,100,1000,500]] 
 \end{align*}
 For the rejector model, and for the expert confidence model used for Staged we use the parameters $[100,100,1000,500]$. The error bars in the plot are produced by repeating the training process 10 times and obtaining standard deviations to average over the randomness in training. We used a rather simple network architecture so that we can more easily illustrate the model complexity gap, as more complex architectures can easily obtain $\sim 100\%$ accuracy on CIFAR and would not allow us to have a more fine-grained analysis of the gap.

\paragraph{Data Trade-Offs.}  We use the model parameters $[100,100,1000,500]$ for all networks in this plot. For each fraction of data labeled, we sample randomly from the training set the corresponding number of points. The error bars are obtained by repeating the training process 10 times for different random samplings of the training set.


\end{document}